\documentclass[onecolumn,11pt,a4paper]{article}

\usepackage[utf8]{inputenc}        
\usepackage{lmodern}               
\usepackage[T1]{fontenc}           

\usepackage[margin=1.20in,top=1.35in,bottom=1.35in]{geometry}  
\usepackage{latexsym}              
\usepackage{amsmath}               
\usepackage{amssymb}               
\usepackage{amsfonts}              
\usepackage{amsthm}                
\usepackage{stmaryrd}
\usepackage{amsbsy}
\allowdisplaybreaks
\usepackage{bbm}
\usepackage{tabularx}
\usepackage{multirow,booktabs}
\usepackage{array}

\usepackage{nicematrix}

\usepackage{tikz}                  
\usetikzlibrary{arrows,positioning,shapes}
\usetikzlibrary{shapes.multipart}
\usetikzlibrary{arrows.meta}
\usetikzlibrary{decorations.pathreplacing}
\usetikzlibrary{arrows,intersections}

\usepackage{enumerate}             
\usepackage{cmap}                  
\usepackage{color}                     

\definecolor{yellow}{rgb}{1,1,0}
\definecolor{lightgreen}{rgb}{0.5,1,0.5}
\definecolor{darkgray}{rgb}{0.15,0.15,0.15}   
\definecolor{lightgray}{rgb}{0.94,0.94,0.94}  
\definecolor{lightlightgray}{rgb}{0.97,0.97,0.97}  
\definecolor{lightlightblue}{rgb}{0.95,0.97,1.00}  
\definecolor{darkred}{rgb}{0.80,0.00,0.00}    
\definecolor{darkgreen}{rgb}{0.00,0.70,0.00}    
\definecolor{darkblue}{rgb}{0.00,0.00,0.70}    
\definecolor{violet}{rgb}{0.35,0.00,0.70}

\definecolor{lightblue}{cmyk}{0.25,0.15,0,0}
\definecolor{lightred}{cmyk}{0,0.25,0.15,0}
\definecolor{lightgreen}{cmyk}{0.15,0,0.25,0}
\definecolor{lightyellow}{cmyk}{0,0,0.25,0}

\definecolor{darkyellow}{rgb}{0.9,0.9,0.00}    
\definecolor{orange}{rgb}{1.00,0.50,0.00}    

\usepackage[colorlinks, plainpages=false,pdfcenterwindow=true,
    pdftoolbar=false,pdfmenubar=false,
    pdftitle={Overview of AdaBoost: Reconciling its views to better understand its dynamics},
    pdfauthor={Perceval Beja-Battais (ENS Paris-Saclay)},
    linkcolor=red,citecolor=blue,filecolor=black,urlcolor=violet]{hyperref}

\usepackage{graphicx}              
\graphicspath{{fig/}}              
\usepackage{caption}               
\usepackage{subcaption}            
\usepackage{multicol}              
\usepackage{float}                 
\usepackage{framed}                
\usepackage{palatino}              
\usepackage{enumitem}              
\usepackage{wasysym}               
\usepackage{marvosym}               

\usepackage{algpseudocode}

\usepackage{mdframed}
\usepackage{pgfplots}
    \usepgfplotslibrary{fillbetween}

\usepackage[bottom]{footmisc}

\usepackage[nameinlink]{cleveref}
\theoremstyle{plain}  
\newtheorem{thm}{Theorem}[subsection]
\newtheorem{defn}[thm]{Definition}

\newtheorem{prop}[thm]{Proposition}

\theoremstyle{remark}  
\newtheorem{rem}{Remark}

\algnewcommand\algorithmicforeach{\textbf{for each}}
\algdef{S}[FOR]{ForEach}[1]{\algorithmicforeach\ #1\ \algorithmicdo}

\renewcommand{\Cross}{$\mathbin{\tikz [x=1.4ex,y=1.4ex,line width=.2ex, purple] \draw (0,0) -- (0.5,0.5) (0,0.5) -- (0.5,0);}$}%
\renewcommand{\Square}{$\mathbin{\tikz [x=1.4ex,y=1.4ex,line width=.2ex, brown] \draw (0,0) -- (0.5,0) -- (0.5,0.5) -- (0,0.5) -- (0,0);}$}

\makeatletter
\renewenvironment{proof}[1][\proofname]{\par
    \pushQED{\qed}%
    \normalfont \topsep6\p@\@plus6\p@\relax
    \list{}{\leftmargin=1.5em
          \rightmargin=\leftmargin
          \settowidth{\itemindent}{\itshape#1}%
        \labelwidth=\itemindent
          \parsep=0pt \listparindent=\parindent
    }
    \item[\hskip\labelsep
        \itshape
    #1\@addpunct{.}]\ignorespaces
}{%
    \popQED\endlist\@endpefalse
}

\newenvironment{definition}[1][\definitionname]{\par
    \begin{mdframed}[backgroundcolor=gray!5,rightline=false,leftline=true, topline=false, bottomline=false, linecolor=darkblue]
        \begin{defn}[#1]
}{
    \end{defn}\end{mdframed}
}

\newenvironment{theorem}[1][\definitionname]{\par
    \begin{mdframed}[backgroundcolor=gray!5,rightline=false,leftline=true, topline=false, bottomline=false, linecolor=darkred]
        \begin{thm}[#1]
}{
    \end{thm}\end{mdframed}
}

\newenvironment{proposition}[1][\definitionname]{\par
    \begin{mdframed}[backgroundcolor=gray!5,rightline=false,leftline=true, topline=false, bottomline=false, linecolor=violet]
        \begin{prop}[#1]
}{
    \end{prop}\end{mdframed}
}

\newenvironment{remark}{\par
    \begin{mdframed}[backgroundcolor=gray!5,rightline=false,leftline=true, topline=false, bottomline=false, linecolor=black]
        \begin{rem}
}{
    \end{rem}\end{mdframed}
}

\newenvironment{alg}[1][\definitionname]{\par
\begin{mdframed}[backgroundcolor=blue!4,rightline=false,leftline=false, topline=false, bottomline=false, linecolor=darkblue, innerbottommargin=-0.1cm, innertopmargin=-0.1cm, 
    innerleftmargin=-0.01cm, innerrightmargin=-0.01cm]
\begin{algorithm}[H]
\caption{#1}
}{
    \end{algorithm}\end{mdframed}
}

\makeatother

\mdfdefinestyle{leftbar}{%
    linewidth=1pt,
    linecolor=black,
    topline=false,
    bottomline=false,
    rightline=false,
    innertopmargin=0pt,
    innerbottommargin=0pt,
    innerrightmargin=0pt,
    innerleftmargin=0pt,
    leftmargin=0pt,
    rightmargin=0pt,
    skipabove=0pt,
    skipbelow=0pt,
    nobreak=true
}

\mdfdefinestyle{lightleftbar}{
    linewidth=1pt,
    linecolor=lightlightblue,
    topline=false,
    bottomline=false,
    rightline=false,
    innertopmargin=0pt,
    innerbottommargin=0pt,
    innerrightmargin=0pt,
    innerleftmargin=0pt,
    leftmargin=0pt,
    rightmargin=0pt,
    skipabove=0pt,
    skipbelow=0pt,
    nobreak=true,
}

\definecolor{theorembar}{rgb}{0.95,0.00,0.00}      
\definecolor{defnbar}{rgb}{0.00,0.00,0.95}         
\definecolor{propositionbar}{rgb}{0.58,0.00,1.00}  
\definecolor{corollarybar}{rgb}{0.00,0.95,0.95}    
\definecolor{lemmabar}{rgb}{1.00,0.00,0.78}        
\definecolor{warningbar}{rgb}{0.90,0.00,0.00}      
\definecolor{propertybar}{rgb}{1.00,1.00,0.00}  
\definecolor{examplebar}{rgb}{0.00,0.00,0.00}      
\definecolor{remarkbar}{rgb}{1.00,0.53,0.00}       

\usepackage{lastpage,fancyhdr}     
\providecommand*{\hr}[1][class-arg]{%
    \hspace*{\fill}\hrulefill\hspace*{\fill}
    \vskip 0.65\baselineskip
}

\usepackage{xcolor}
\usepackage{algorithm2e}

\SetKwInput{KwInput}{Input}                
\SetKwInput{KwOutput}{Output}              

\RestyleAlgo{boxruled}

\newcommand{\X}{\mathcal{X}}
\newcommand{\Y}{\mathcal{Y}}
\newcommand{\R}{\mathbb{R}}
\renewcommand{\span}{\mathrm{Span}}
\newcommand{\D}{\mathcal{D}}
\renewcommand{\P}{\mathbb{P}}
\newcommand{\E}{\mathbb{E}}
\newcommand{\A}{\mathcal{A}}
\renewcommand{\H}{\mathcal{H}}
\newcommand{\N}{\mathbb{N}}

\DeclareMathOperator*{\argmax}{arg\,max}
\DeclareMathOperator*{\argmin}{arg\,min}

\usepackage[toc,title,page]{appendix}

\usepackage{biblatex}[backend=bibtex]
\title{Overview of AdaBoost: Reconciling its views to better understand its dynamics}
\author{%
  Perceval Beja-Battais \\
  Centre Borelli\\
  ENS Paris-Saclay, Université Paris-Saclay\\
  91190 Gif-sur-Yvette\\
  \texttt{\href{mailto:perceval.beja-battais@ens-paris-saclay.fr}{perceval.beja-battais@ens-paris-saclay.fr} } \\
}

\begin{document}

\maketitle

{
  \hypersetup{linkcolor=black}
  \tableofcontents
}

\newpage

\begin{abstract}
Boosting methods have been introduced in the late 1980's. They were born following the theoritical aspect of PAC learning. The main idea 
of boosting methods is to combine weak learners to obtain a strong learner. The weak learners are obtained iteratively by an heuristic 
which tries to correct the mistakes of the previous weak learner. In 1995, Freund and Schapire~\cite{freund_decision-theoretic_1997-1} 
introduced AdaBoost, a boosting algorithm that is still widely used today. Since then, many views of the algorithm have been proposed 
to properly tame its dynamics. In this paper, we will try to cover all the views that one can have 
on AdaBoost. We will start with the original view of Freund and Schapire before covering the different views and
unify them with the same formalism.
We hope this paper will help the non-expert reader to better understand the dynamics of AdaBoost and how the different
views are equivalent and related to each other.\vspace{1cm}\newline
\end{abstract}

\section{Introduction, problematic \& notations}

\subsection{Introduction}

Machine learning is a dynamic field that encompasses a wide range of algorithms and methodologies, aiming to enable computational systems to 
learn from data and make predictions or decisions without being explicitly programmed. The growth of computational power, the availability of 
vast amounts of data, and advancements in statistical and algorithmic techniques have propelled the field of machine learning to new heights~\cite{lecun2015deep, jordan2015machine}. 
The application of machine learning methods are ubiquitous in our daily lives, from the personalized recommendations on internet platforms~\cite{wei2017collaborative, steck2021deep, smith2017two} to 
medicine~\cite{chen2018rise, vamathevan2019applications, cruz2006applications}, finance~\cite{dixon2020machine, emerson2019trends, culkin2017machine, gogas2021machine}, 
or autonomous driving~\cite{kiran2021deep, uccar2017object, mozaffari2020deep}. One prominent subset of machine learning methods that has gained considerable 
attention and achieved remarkable success in various applications is boosting~\cite{freund_decision-theoretic_1997-1, schapire2013boosting}.

At its core, machine learning involves the development of models that can automatically extract patterns and insights from data. These models learn from experience,
 iteratively refining their performance through exposure to labeled examples. By analyzing patterns and relationships within the data, machine learning algorithms 
 can make accurate predictions on new, unseen instances, thereby uncovering hidden knowledge and informing decision-making processes.

Boosting algorithms are a class of machine learning methods (more precisely of ensemble methods~\cite{dietterich2000ensemble, zhou2014ensemble, sagi2018ensemble}) that aim to enhance the performance of weak learners by iteratively combining their predictions 
to create a more powerful and accurate model~\cite{freund_decision-theoretic_1997-1, schapire2013boosting}. These methods were introduced in the 1990's and have since emerged as a cornerstone of contemporary machine learning research. 
Boosting algorithms iteratively train a sequence of weak models, each focusing on the instances that were previously misclassified or had the largest prediction errors. 
By assigning higher weights to these challenging instances, subsequent models within the boosting framework can effectively address their misclassification and improve overall accuracy.

One of the key advantages of boosting methods is their ability to adaptively allocate resources to challenging examples, thereby mitigating the impact of noise and outliers. 
By emphasizing the instances that are difficult to classify, boosting algorithms excel at capturing complex decision boundaries and achieving high predictive performance. 
Furthermore, boosting methods are versatile and can be applied to a variety of learning tasks, including classification, regression, and ranking problems.

Boosting algorithms come in various flavors, each with its unique characteristics and strengths. One of the earliest and most influential boosting methods is AdaBoost~\cite{freund_decision-theoretic_1997-1}, 
short for Adaptive Boosting. AdaBoost iteratively adjusts the weights of misclassified instances, with subsequent models paying greater attention to these misclassified examples. 
Another popular boosting method is Gradient Boosting~\cite{natekin2013gradient, bentejac2021comparative, friedman2002stochastic}, which leverages the concept of gradient descent to optimize a loss function by adding weak models in a stage-wise manner. 
XGBoost~\cite{chen2016xgboost, ogunleye2019xgboost, dhaliwal2018effective} and LightGBM~\cite{ke2017lightgbm, sun2020novel, wang2017lightgbm} are notable implementations of gradient boosting 
algorithms that have gained widespread adoption due to their scalability and high performance.

\subsection{Problematic}
Boosting methods have emerged for the first ones in the early 1990's, and the proper modern formalization of boosting 
methods has been made in 1990 by Schapire~\cite{schapire1990strength}. At that time, boosting methods were seen as
PAC learning algorithms. Later, Freund and Schapire~\cite{freund1995boosting} collaborated to work deeper on boosting algorithms, and 
in 1995, they introduced AdaBoost~\cite{freund_decision-theoretic_1997-1}, which is still widely used today. Since then, many researchers
have tried to understand the dynamics of AdaBoost which were not sufficiently understood. To this day, there is still a lot that
we do not know about AdaBoost, especially about the convergence properties of the algorithm when we increase the number of weak learners~\cite{belanich2012convergence}. 
The goal of this paper was in the first place to try to answer the question of the cyclic behavior of AdaBoost which was proven in some cases 
in 2004 by Rudin et al.~\cite{rudin2004dynamics}, and that can be observed in many toy examples (see Fig.~\ref{fig:AdaBoost_diversity}). However, the same authors addressed an open problem in 2012~\cite{rudin2012open} for the general
case: \textit{Does AdaBoost always cycle?} Some answers have been found very recently, in 2023, by Belanich and Ortiz~\cite{belanich2012convergence}.
Indeed, by studying the ergodic dynamics of AdaBoost, the authors have shown this conjecture as an intermediate result, which we will quickly
present in the last subsection.

Yet, this overview aims to cover all views that one can have on AdaBoost. We tried to use the same formalism for all views, to try and unify them.

\subsection{Notations}
In all what follows, we will consider a classification problem. We will denote by $\X$ the input space, and by $\Y$ the output space.
$\X$ is a compact subset of $\R^d$, and $\Y$ is a finite set, which will be $\{-1, 1\}$ if the problem is binary, 
and $\{1,\dots,K\}$ if the problem is multiclass. The sequence of weights produced by AdaBoost will be denoted $W_0,\dots,W_{T-1}$,
and the sequence of classifiers will be denoted $h_1,\dots,h_T$, where $T$ is the number of iterations of the algorithm. 
The final classifier will be denoted $H$. We will denote by $\mathbbm{1}_A$ the indicator function of the set $A$. For binary problems, we will
use the notation $\eta$ or $\eta_h$ to denote the vector $(y_1 h(x_1), \dots, y_m h(x_m))$ which we call a dichotomy. $\eta_h(i)$ is 1 if 
the classifier $h$ is right, and $-1$ if it is wrong. Also, we will denote by $\lambda$ the standard Lebesgue measure.

\section{Views of AdaBoost}
\subsection{The original view: a PAC learning algorithm}\label{sec:AdaBoost_as_a_PAC_learning_algorithm}
\subsubsection{What is a PAC learning algorithm?}
\begin{definition}[PAC learning algorithm]
    In our setting, we say that a hypothesis space $\H$ is PAC learnable \textbf{if there exists an algorithm $A$} such that:
    \begin{itemize}
        \item For any distribution $D$ over $\X$, and for any $0 < \epsilon, \delta < \frac{1}{2}$, the algorithm $A$ outputs a classifier $h \in \H$ such that with
        probability at least $1-\delta$, we have:
        \begin{equation}
            \P_{(x,y) \sim D}(h(x) \neq y) \leq \epsilon
        \end{equation}
        \item The number of examples $m$ needed to achieve this bound is polynomial in $\frac{1}{\epsilon}$, $\frac{1}{\delta}$ and $d$.
    \end{itemize}
\end{definition}

\subsubsection{AdaBoost: the original formulation} 
The first view, when AdaBoost was introduced, was to see it as a PAC learning algorithm. The idea was to create a boosting algorithm
that iteratively takes into accound the errors made by the previous classifiers, and to try to correct them. 
Freund and Schapire~\cite{freund_decision-theoretic_1997-1} proposed the following algorithm, which is the original formulation of AdaBoost. They
also verified that the algorithm is a PAC learning algorithm, and give bounds on the number of examples needed to achieve a given
accuracy. The pseudo-code implementation of the algorithm is given in Alg.~\ref{alg:Original_AdaBoost_discrete}.

\begin{alg}[Original formulation of AdaBoost (discrete) for a binary classification problem]\label{alg:Original_AdaBoost_discrete}
\KwData{$(x_1, y_1),\dots, (x_m, y_m) \in \X \times \{-1, 1\}$}
Initialize $W_0(i) = \frac{1}{m}$ for $i = 1,\dots,m$ \\
\For{$t = 1,\dots,T$}
    {
    Train a weak learner $h_t : \X \rightarrow \{-1, 1\}$ w.r.t.~the distribution $W_{t-1}$ \\
    Compute $\displaystyle \epsilon_t = \sum_{i=1}^m W_{t-1}(i) \mathbbm{1}_{h_t(x_i) \neq y_i}$ the weighted error of $h_t$ \\
    Compute $\displaystyle \alpha_t = \frac{1}{2} \log \frac{1 - \epsilon_t}{\epsilon_t}$ \\
    Update $\displaystyle W_t(i) = \frac{W_{t-1}(i)\exp(-\alpha_t y_i h_t(x_i))}{Z_t}$ for $i = 1,\dots,m$, where $\displaystyle Z_t = \sum_{i=1}^{m} W_{t-1}(i)\exp(-\alpha_t y_i h_t(x_i))$
    }
\Return $\displaystyle H: x\in \X \mapsto \mathrm{sign}\left(\sum_{t=1}^T \alpha_t h_t(x)\right) \in \Y$
\end{alg}

\begin{remark}
    Freund and Schapire designed this algorithm for weak learners, i.e.~classifiers that have an error rate slightly better than random guessing.
    They proved that if the weak learners error rates are $\epsilon_1, \dots, \epsilon_T$, then the error rate of the final classifier is bounded by
    \begin{equation}
        \epsilon \leq 2^T \prod_{t=1}^T \sqrt{\epsilon_t(1 - \epsilon_t)}
    \end{equation}
\end{remark}

\begin{remark}
    The first nad most natural set of classifiers that come to mind when we think about weak learners are decision stumps, i.e.~classifiers that
    are constant on a half-space, and constant on the other half-space. Most of AdaBoost applications in practice use decision stumps as weak learners.
    We also can think of decision trees of an arbitrary depth as weak learners, but in practice, they may not be used because they are slower to train.
\end{remark}

\begin{remark}
    The first few things we can observe are that first the final classifier does not take in consideration the potential relevance 
    (to be defined) of each classifier in the final vote (the weights of the classifiers are all equal to $1$), and second that 
    the sequence of weights are updated only considering the error made by the previous classifier, and not the ones before. 
    We can easily see that both sequence of weights and classifiers are order 1 Markov chains, therefore we can legitimately 
    ask ourselves if the algorithm can get stuck in cycles, which can be the case in some examples~\cite{rudin_dynamics_2003}.
\end{remark}

However, at that time, Freund and Schapire did not give any insight of how they chose their updating rule for the weights.
They just said that they wanted to correct the errors made by the previous classifier, and that they wanted 
to give more importance to the examples that were misclassified. In short, they designed a
heuristic algorithm, and proved that it was a PAC learning algorithm. Later on, it appeared that
AdaBoost does not only come from a heuristic, but also from a theoritical point of view, as we will see in the next sections.

Freund and Schapire also gave a multiclass version of AdaBoost in 1996~\cite{freund1996experiments}, which base on the same heuristic.
We give the algorithm in Alg.~\ref{alg:Original_AdaBoost_discrete_multiclass}.

\begin{alg}[Original formulation of AdaBoost (discrete) for a multiclass classification problem]\label{alg:Original_AdaBoost_discrete_multiclass}
\KwData{$(x_1, y_1),\dots, (x_m, y_m) \in \X \times \{1,\dots,K\}$}
Initialize $W_0(i) = \frac{1}{m}$ for $i = 1,\dots,m$ \\
\For{$t=1,\dots,m$}{
    Train a weak learner $h_t : \X \rightarrow \{1,\dots,K\}$ w.r.t.~the distribution $W_{t-1}$ \\
    Compute $\epsilon_t = \sum_{i=1}^m W_{t-1}(i) \mathbbm{1}_{h_t(x_i) \neq y_i}$ the weighted error of $h_t$ \\
    Compute $\alpha_t = \frac{1 - \epsilon_t}{\epsilon_t}$ \\
    Update $W_t(i) = \frac{W_{t-1}(i)(\alpha_t\mathbbm{1}_{h_t(x_i) = y_i} + \mathbbm{1}_{h_t(x_i)\ne y_i})}{Z_t}$ for $i = 1,\dots,m$, where $Z_t = \sum_{i=1}^{m} W_{t-1}(i)\exp(-\alpha_t \mathbbm{1}_{h_t(x_i) = y_i})$
    }
\Return $H: x\in \X \mapsto \argmax_{y \in \Y} \sum_{t=1}^T \left(\log\frac{1}{\alpha_t}\right) \mathbbm{1}_{h_t(x) = y}$
\end{alg}

\subsubsection{Real AdaBoost}
The original formulation of AdaBoost is a discrete version of the algorithm. However, it is possible to define a real version of AdaBoost,
as Freund and Schapire did~\cite{freund1996experiments}. The main difference is that the weak learners are not restricted to output
only $\{-1, 1\}$, but are functionals $h:X\times Y \rightarrow [0,1]$. 
\begin{definition}[Pseudo-loss and Update Rule for real AdaBoost]
    To each couple $(x,y)$, we associate a \textit{plausibility} $h(x,y)$ that $x$ belongs to the class $y$ (plausibility instead of probability, because it is not one). 
    We then define the \textit{pseudo-loss} $\epsilon_t$ of $h_t$ at iteration $t$ as:
    \begin{equation}
        \epsilon_t = \frac{1}{2} \sum_{(i,y) : y \ne y_i} W_{t-1}(i,y) (1 - h_t(x_i, y_i) + h_t(x_i, y))
    \end{equation}
    where $W_{t-1}(i,y)$ is the weight of the couple $(x_i, y)$ at iteration $t-1$.
    Finally, we can define the weight update rule, following the same principle, as:
    \begin{equation}
        \begin{aligned}
            W_0(i,y) &= \frac{1}{mK}, \quad \forall i \in \{1,\dots,m\}, \forall y \in \Y \\
            W_t(i,y) &= \frac{W_{t-1}(i,y) \alpha_t^{\frac{1}{2}(1 - h_t(x_i, y_i) + h_t(x_i, y))}}{Z_t}
        \end{aligned}
    \end{equation}
    where $Z_t$ is a normalization factor.
\end{definition}
The algorithm now writes as in Alg.~\ref{alg:Original_AdaBoost_real}.

\begin{alg}[Original formulation of AdaBoost (real) for a multiclass classification problem]\label{alg:Original_AdaBoost_real}
\KwData{$(x_1, y_1),\dots, (x_m, y_m) \in \X \times \Y$}
Initialize $W_0(i,y) = \frac{1}{mK}$ for $i = 1,\dots,m$, $y \in \Y$ \\
\For{$t=1,\dots,T$}{
    Train a weak learner $h_t:\X \times \Y \rightarrow [0,1]$ w.r.t.~the distribution $W_{t-1}$ \\
    Compute $\epsilon_t = \frac{1}{2} \sum_{(i,y) : y \ne y_i} W_{t-1}(i,y) (1 - h_t(x_i, y_i) + h_t(x_i, y))$ \\
    Compute $\alpha_t = \frac{1 - \epsilon_t}{\epsilon_t}$ \\
    Update $W_t(i,y) = \frac{W_{t-1}(i,y) \alpha_t^{\frac{1}{2}(1 - h_t(x_i, y_i) + h_t(x_i, y))}}{Z_t}$ \\
    }
\Return{$H: x\in \X \mapsto \argmax_{y \in \Y} \sum_{t=1}^T \left(\log\frac{1}{\alpha_t}\right) h_t(x, y)$}
\end{alg}

\subsection{AdaBoost as successive optimization problems}\label{sec:Optimization_views_of_AdaBoost}
\subsubsection{AdaBoost as a gradient descent}\label{sec:AdaBoost_as_a_gradient_descent_algorithm}

In 1999, Mason et al.~\cite{mason1999boosting} proposed a new view of boosting. They saw any boosting algorithm, 
including AdaBoost, as a gradient descent on the set of linear combinations of classifiers in $\H$.
\begin{definition}[Optimization problem for AdaBoost as a gradient descent]\label{def:Optimization_problem_for_AdaBoost_as_a_gradient_descent}
    Let $\langle, \rangle$ be an inner product on $\span(\H)$. Define a \textit{cost function}
    $C: \span(\H) \rightarrow \R$. We define the optimization problem associated to the cost function $C$ as:
    \begin{equation}
        \min_{H \in \span(\H)} C(H)
    \end{equation}
\end{definition}

\begin{definition}[Margin and margin cost-functionals]
    We define the margin of a classifier $h \in \H$ as
    \begin{equation}
        \gamma(h) = \langle h(x), y \rangle = \frac{1}{m}\sum_{i=1}^m h(x_i)y_i.
    \end{equation}
    and the margin cost-functionals as
    \begin{equation}
        C(h) = \frac{1}{m}\sum_{i=1}^{m} c(h(x_i)y_i)
    \end{equation}
    where $c$ is a differentiable function of the margin $\gamma$.
\end{definition}
We can now prove Thm.~\ref{thm:AdaBoost_as_a_gradient_descent_algorithm}.

\begin{theorem}[AdaBoost as a gradient descent]\label{thm:AdaBoost_as_a_gradient_descent_algorithm}
    AdaBoost is a gradient descent algorithm on the optimization problem defined in Def.~\ref{def:Optimization_problem_for_AdaBoost_as_a_gradient_descent} for the
    margin-cost functional $C$ with $c(\gamma) = \exp{-\gamma}$.
\end{theorem}

\begin{proof}
At each iteration $t$, we will denote by $H_t$ the current resulting classifier, which writes
\begin{equation}
    H_t = \sum_{i=1}^t \alpha_i h_i
\end{equation}
Now, consider the inner product (recall that $m$ is the number of training examples)
\begin{equation}
    \langle h,g \rangle = \frac{1}{m}\sum_{i=1}^m h(x_i)g(x_i).
\end{equation}
For AdaBoost, fix $c(\gamma) = \exp(-\gamma)$, thus the cost function we want to optimize over the set of classifiers $\H$ is
\begin{equation}
    C(h) = \frac{1}{m}\sum_{i=1}^{m} \exp(-y_i h(x_i))
\end{equation}
To find the next classifier in the sequence, we have to find the classifier $h_t$ 
that minimizes the weighted error $\epsilon_t$. With a reasonable cost function
$c$ of the margin (i.e.~monotonic decreasing), it is equivalent to finding the classifier
$h_t$ that maximizes $-\langle \nabla C(h_{t-1}), h_t \rangle$. Indeed, we have
on the one hand
\begin{equation}
  \nabla C(h) = -\frac{1}{m}y_i e^{-y_i h(x_i)}\mathbbm{1}_{x = x_i}
\end{equation}
which leads to
\begin{equation}
  \begin{aligned}
    -\langle \nabla C(H_{t-1}), h_t \rangle &= \frac{1}{m^2}\sum_{i=1}^m y_i h_t(x_i) e^{-y_i H_{t-1}(x_i)} \\
    &\propto \frac{1}{m^2}\sum_{i=1}^m y_i h_t(x_i) W_{t-1}(x_i) \\
  \end{aligned}
\end{equation}
up to a normalization constant. That means that maximizing $-\langle \nabla C(H_{t-1}), h_t \rangle$ (i.e.~gradient descent step) is equivalent to
minimizing $\sum_{i=1}^m W_{t-1}(x_i) \mathbbm{1}_{h_t(x_i) \neq y_i}$, which is the weighted error of $h_t$. Thus,
the update rule for this gradient descent is the same as the one of AdaBoost.
\end{proof}
That new view of AdaBoost is interesting because it gives both theoritical and intuitive justification of the algorithm. For some
cost function $c$ that ponderates how important it is to classify correctly our observations,
we look for the direction for which the cost function decreases the most, and we take a step in that direction.
We can now formulate a new (equivalent) version of discrete AdaBoost which is written in Alg.~\ref{alg:Gradient_AdaBoost}.

\begin{alg}[AdaBoost as a gradient descent]\label{alg:Gradient_AdaBoost}
\KwData{$(x_1, y_1),\dots, (x_m, y_m) \in \X \times \{-1, 1\}$}
Initialize $W_1(i) = \frac{1}{m}$ for $i = 1,\dots,m$ \\
Initialize $h_0(x) = 0$ \\
\For{$t=1,\dots,T$}{
  Compute $h_t = \argmax_{h \in \H} -\langle \nabla C(h_{t-1}), h\rangle$\\
  Let $\alpha_t = \frac{1}{2} \log \frac{1 - \epsilon_t}{\epsilon_t}$ \\
  Let $H_t = \frac{H_{t-1} + \alpha_t h_t}{\sum_{s=1}^{t} |\alpha_s|}$ \\
  Update $W_t = \frac{c'(y_i H_t(x_i))}{\sum_{i=1}^{m}c'(y_i H_t(x_i))}$ \\
  }
\Return{$H_T$}
\end{alg}

\subsubsection{AdaBoost as an additive model}\label{sec:AdaBoost_as_an_additive_model}
\paragraph{Additive models}\label{subsec:Additive_models}
As seen in the previous subsection, at each iteration the classifier which is produced
by AdaBoost can be written:
\begin{equation}
    H_t = \sum_{i=1}^t \alpha_i h_i
\end{equation}
This is an example of what we call an additive model. A subset of the additive models are the \textit{additive regression models}. 
\begin{definition}[Additive regression models]
These models have the form:
    \begin{equation}
        H(x) = \sum_{i=1}^m h_i(x_i)
    \end{equation}
    where each $h_i$ is a function of only one data point $x_i$.
\end{definition}
\begin{remark}
    A way to find the good functions $h_i$ is to use the \textit{backfitting algorithm}~\cite{friedman1981projection}.
    A backfitting update writes:
    \begin{equation}
        h_i(x_i) = \E\left(y - \sum_{j \ne i} h_j(x_j) | x_i\right)
    \end{equation}
\end{remark}
In our case, for more general additive model as the one we have in AdaBoost, we can use the following update:
\begin{equation}
    (\alpha_t,h_t) = \argmin_{\alpha\in\R, h\in\H} \E\left(C(y - H_{t-1}(x) - \alpha h(x))\right)
\end{equation}
for a given cost function $C$. This update is called the \textit{greedy forward stepwise approach}.

\begin{remark}
    For classification problems, we can use Bayes theorem because all we have to know, in order to
    produce a good classifier, is the conditional probabilities $\P(y=k|x)$ for all classes $k$. 
    For instance, for a binary classification problem, a subset of additive models 
    are \textit{additive logistic binary classification models}, which write:
    \begin{equation}
        H(x) = \sum_{t=1}^T h_t(x) = \log \frac{\P(y=1|x)}{\P(y=-1|x)}
    \end{equation}
    That directly implies that 
    \begin{equation}
        \P(y=1|x) = \frac{e^{H(x)}}{1 + e^{H(x)}}
    \end{equation}
\end{remark}

\paragraph{AdaBoost as an additive model}

In~\cite{friedman2000additive}, Friedman, Hassie and Tibshirani show two main results considering AdaBoost. The first one is the following:
\begin{proposition}[AdaBoost as an additive model]
    Discrete AdaBoost can be seen as an additive logistic regression
    model via Newton updates for the loss $C(\gamma) = e^{-\gamma}$.
\end{proposition}
\begin{proof}
Suppose we are at iteration $t$ of AdaBoost, and that we have already computed
the classifier $H_{t-1}$. We want to find the classifier $h_t$ such that:
\begin{equation}
    \begin{aligned}
        (\alpha_t, h_t) &= \argmin_{\alpha, h} \E\left(e^{-y(\alpha h(x) + H_{t-1}(x))}\right) \\
        &= \argmin_{\alpha, h} \E\left(e^{-\alpha\eta_h}e^{-yH_{t-1}(x)}\right) \\
        &\approx \argmin_{\alpha, h} \E\left(e^{-yH_{t-1}(x)}\left(1-\alpha \eta_h + \frac{\alpha^2}{2}\right)\right) \\
    \end{aligned}
\end{equation}
since $y^2h^2(x) = 1$ and where we did a Taylor expansion at the last line (recall from the notations subsection that $\eta_h = yh(x)$).
Now, to minimize both in $\alpha_t$ and $h_t$, we will first fix $\alpha_t = \alpha$ and minimize over $h_t$,
then take this minimizer $h_{t,\alpha}$ to minimize over $\alpha_t$.\newline
First, we have
\begin{equation}
    h_{t,\alpha}(x) = \argmin_h \E\left(e^{-yH_{t-1}(x)}(1-\alpha \eta_h + \frac{\alpha^2}{2})\right)
\end{equation}
which has a solution which is independent of $\alpha$:
\begin{equation}
    h_t(x) = \begin{cases}
            1 & \mbox{if } \E\left(e^{-yH_{t-1}(x)}y | x\right) > 0\\
            -1 & \mbox{otherwise.}
        \end{cases}
\end{equation}
Then, to minimize over $\alpha$, we can show that:
\begin{equation}
    \begin{aligned}
        \alpha_t &= \argmin_\alpha \E\left(e^{-yH_{t-1}(x)}e^{-\alpha\eta_{h_t}}\right) \\
        &= \frac{1}{2} \log \frac{1 - \epsilon_t}{\epsilon_t}
    \end{aligned}
\end{equation}
where $\epsilon_t = \E\left(e^{-yH_{t-1}(x)}\mathbbm{1}_{h(x) \neq y}\right)$ is the weighted error of $h_t$.
Finally, this update rule writes:
\begin{equation}
    \begin{aligned}
        H_t(x) &= H_{t-1}(x) + \frac{1}{2} \log \frac{1 - \epsilon_t}{\epsilon_t} h_t(x) \\
        W_t(x) &= W_{t-1}(x) e^{-\alpha_t \eta_{h_t}}
    \end{aligned}
\end{equation}
We can notice that this update rule is the same as the one in the original formulation of AdaBoost.
\end{proof}

The second result of~\cite{friedman2000additive} is the following. 
\begin{proposition}[Real AdaBoost as an additive model]
Real AdaBoost can be seen as an additive logistic regression
by stagewise and approximate optimization of $C(\gamma) = e^{-\gamma}$. 
\end{proposition}
\begin{proof}
    Again, we start at the iteration $t$ of (real) AdaBoost, and we have already computed the classifier $H_{t-1}$. We look
    for the classifier $h_t$ such that:
    \begin{equation}
        \begin{aligned}
            h_t &= \argmin_{\alpha, h} \E\left(e^{-y( h(x) + H_{t-1}(x))}|x\right) \\
            &= \argmin_{h} e^{-h(x)}\E\left(e^{-yH_{t-1}(x)} \mathbbm{1}_{y=1}|x\right) + e^{h(x)}\E\left(e^{-yH_{t-1}(x)} \mathbbm{1}_{y=-1}|x\right) \\
        \end{aligned}
    \end{equation}
    Setting the derivatives w.r.t.~$h_t$ to $0$, we get:
    \begin{equation}
        h_t(x) = \frac{1}{2} \log \frac{\E\left(e^{-yH_{t-1}(x)} \mathbbm{1}_{y=1}|x\right)}{\E\left(e^{-yH_{t-1}(x)} \mathbbm{1}_{y=-1}|x\right)}
    \end{equation}
    Thus, the weight update rule writes:
    \begin{equation}
        W_t(x) = W_{t-1}(x) e^{-\eta_{h_t}}
    \end{equation}
    which is the same as the one in the original formulation of real AdaBoost. 
\end{proof}
\begin{remark}
The authors from~\cite{friedman2000additive} also propose a new algorithm, based on what
we have just shown, that uses as the loss function the \textit{binomial log-likelihood} 
$C(\gamma) = - \log(1 + e^{-2\gamma})$. The algorithm is called \textit{LogitBoost}.
It has a similar form as AdaBoost because with a Taylor expansion of the loss function at $F(x) = 0$,
we have, with $\gamma = yF(x)$:
\begin{equation}
    - \log(1 + e^{-2yF(x)}) \approx e^{-yF(x)} + \log 2 - 1
\end{equation}
\end{remark}

\subsubsection{AdaBoost as an entropy projection}\label{sec:AdaBoost_as_an_entropy_projection}
This view has been proposed in 1999 by Kivinen and Warmuth~\cite{kivinen_boosting_1999}. In relies on the 
idea that the corrective updates of AdaBoost, and more generally speaking of boosting, can be seen 
as a solution to a relative entropy (constrained) minimization problem.\newline

\paragraph{Relative entropy minimization theorem}
\begin{definition}[Relative entropy]
Define the relative entropy between two distributions $p$ and $q$ as the Kullback-Leibler divergence:
\begin{equation}
    KL(p||q) = \sum_{i=1}^{m} p_i \log \frac{p_i}{q_i}
\end{equation}
\end{definition}
At every iteration $t$ of AdaBoost, we want to find a weight distribution $W_t$ that is
not correlated with the mistakes from the previous classifier $h_t$ that learnt on $W_{t-1}$. 
We can force this by introducing the constraint $W_t^T h_t(x) = 0$. Then, the updated distribution
$W_t$ should not differ too much from the previous distributions in order to take into
account the knowledge we already have from the data, and also in order not to react too much
to the possible noise. To do so, we want to minimize the relative entropy between $W_{t-1}$
and $W_t$ under the constraint $W_t^T h_t(x) = 0$. 
\begin{definition}[Relative entropy minimization problem]
    The optimization problem writes:
    \begin{equation}
        \begin{aligned}
            \min_{W\in \Delta^m} &\sum_{i=1}^{m} W(i) \log \frac{W(i)}{W_{t-1}(i)} \\
            \text{s.t. } &W^T \eta_{h_t} = 0
        \end{aligned}
    \end{equation}
\end{definition}
Kivinen and Warmuth~\cite{kivinen_boosting_1999} showed that minimizing the relative entropy
is the same problem as maximizing $-\log Z_t(\alpha)$, where 
\begin{equation}
    Z_t(\alpha) = \sum_{i=1}^{m} W_{t-1}(i) e^{-\alpha y_i h_t(x_i)}
\end{equation}
In short, we have the following result:
\begin{theorem}[Relative entropy minimization theorem]
    \begin{equation}\label{eq:3.1}
        \min_{W \in \Delta^m : W^T yh_t(x) = 0} \sum_{i=1}^{m} W(i) \log \frac{W(i)}{W_{t-1}(i)} = \max_{\alpha\in\R} -\log Z_t(\alpha)
    \end{equation}
    Furthermore, if $W_t, \alpha_t$ are the two solutions of these problems, we have that
    \begin{equation}
        W_t(i) = \frac{W_{t-1}(i) e^{-\alpha_t y_i h_t(x_i)}}{Z_t(\alpha_t)}, \quad \forall i \in \{1,\dots,m\}
    \end{equation}
\end{theorem}
\begin{proof}
    Let us prove those two results. Let $\mathcal{L}(W, \alpha)$ be the Lagrangian of the optimization problem:
    \begin{equation}
        \mathcal{L}(W, \alpha) = \sum_{i=1}^{m} W(i) \log \frac{W(i)}{W_{t-1}(i)} + \alpha W^T \eta_{h_t}
    \end{equation}
    
    First we show that the following minimax equation holds:
    \begin{equation}
        \min_{W \in \Delta^m} \max_{\alpha\in\R} \mathcal{L}(W, \alpha)  = \max_{\alpha\in\R} \min_{W \in \Delta^m} \mathcal{L}(W, \alpha)
    \end{equation}
    We always have that (see~\cite{fan1953minimax}, or convice yourself by the quote "the tallest of the dwarfs is shorter than the shortest of the giants"):
    \begin{equation}
        \max_{\alpha\in\R} \min_{W \in \Delta^m} \mathcal{L}(W, \alpha) \leq \min_{W\in\Delta^m} \max_{\alpha\in\R} L(W, \alpha)
    \end{equation}
    Now, on the one side we have:
    \begin{equation}
        \max_{\alpha\in\R} \min_{W \in \Delta^m} \mathcal{L}(W, \alpha) = \min_{W\in\Delta^m} L(W, \alpha^*)
    \end{equation}
    where $\alpha^*=\argmax_{\alpha\in\R} \left(\min_{W\in\Delta^m} L(W,\alpha)\right)$. On the other side, if $W^T \eta_{h_t} = 0$, 
    then $\max_{\alpha\in\R} \mathcal{L}(W, \alpha) = KL(W||W_{t-1})$. Otherwise $\max_{\alpha\in\R} \mathcal{L}(W, \alpha) = \infty$.
    Therefore,
    \begin{equation}\label{eq:orthogonal}
        \max_{\alpha\in\R} \min_{W \in \Delta^m} \mathcal{L}(W, \alpha) = \min_{W \in \Delta^m : W^T \eta_{h_t} = 0} KL(W||W_{t-1})
    \end{equation}
    Thus, using Eq.~\ref{eq:orthogonal}:
    \begin{equation}
        \min_{W\in\Delta^m} \max_{\alpha\in\R} L(W, \alpha) = L(W^*(\alpha^*), \alpha^*) \geq \min_{W\in\Delta^m} L(W, \alpha^*) = \min_{W\in\Delta^m : W^T \eta_{h_t} = 0} KL(W||W_{t-1})
    \end{equation}
    That implies the other inequality:
    \begin{equation}
        \max_{\alpha\in\R} \min_{W \in \Delta^m} \mathcal{L}(W, \alpha) \geq \min_{W\in\Delta^m} \max_{\alpha\in\R} L(W, \alpha)
    \end{equation}
    and we have equality between these two terms.\newline
    Now we show Eq.~\ref{eq:3.1}. We need to show that $\min_{W \in \Delta^m} \mathcal{L}(W, \alpha) = -\log Z_t(\alpha)$.
    As $KL(.||W_{t-1})$ is a convex differentiable function, we can obtain the minimum by setting its gradient to $0$:
    \begin{equation}
    \begin{aligned}
        &\nabla_W \mathcal{L}(W, \alpha) = \nabla_W KL(W||W_{t-1}) + \alpha \nabla_W W^T \eta_{h_t} = 0 \\
        \iff &\nabla_W KL(W||W_{t-1}) = -\alpha \nabla_W W^T \eta_{h_t} \\
        \iff &\nabla_W \sum_{i=1}^{m} W(i) \log\frac{W(i)}{W_{t-1}(i)} = -\alpha \nabla_W W^T \eta_{h_t} \\
    \end{aligned}
    \end{equation}
    Componentwise, this writes:
    \begin{equation}
        \begin{aligned}
            &\forall i \in \{1,\dots,m\}, \log W(i) - \log W_{t-1}(i) + 1 = -\alpha y_i h_t(x_i) \\
            \iff &\forall i \in \{1,\dots,m\}, W(i) \propto W_{t-1}(i) e^{-\alpha y_i h_t(x_i)} \\
        \end{aligned}
    \end{equation}
    We then need to normalize $W$ to obtain the distribution $W_t=\frac{W}{Z_t(\alpha)}$ (note that this distribution depends on $\alpha$).\newline
    We then have:
    \begin{equation}\label{eq:min_W}
        \begin{aligned}
        \min_{W \in \Delta^m} \mathcal{L}(W, \alpha) = KL(W_t||W_{t-1}) &= \sum_{i=1}^{m} W_t(i) \log \frac{W_t(i)}{W_{t-1}(i)} \\
            &= \sum_{i=1}^{m} \frac{W_{t-1}(i) e^{-\alpha y_i h_t(x_i)}}{Z_t(\alpha)} \log \frac{\frac{W_{t-1}(i) e^{-\alpha y_i h_t(x_i)}}{Z_t(\alpha)}}{W_{t-1}(i)} \\
            &= \sum_{i=1}^{m} \frac{W_{t-1}(i) e^{-\alpha y_i h_t(x_i)}}{Z_t(\alpha)} (-\alpha y_i h_t(x_i) -\log{Z_t(\alpha)}) \\
            &= -\log{Z_t(\alpha)} + \alpha \sum_{i=1}^{m} W_{t}(i) y_i h_t(x_i) \\
            &= -\log Z_t(\alpha)
        \end{aligned}
    \end{equation}
    because we have $\sum_{i=1}^{m} W_{t}(i) y_i h_t(x_i) = 0$ by Eq.~\ref{eq:orthogonal} ($W_t^T \eta_{h_t} = 0$). This completes the proof.
\end{proof}

\begin{remark}
    We can also demonstrate that the same result holds (up to the admissible space) if we update our weights according
    to the \textit{totally corrective update} which does not only correct the mistakes from
    the previous classifier, but also the mistakes from all previous classifiers. This update
    is defined optimizing the same problem but with $t$ constraints instead of $1$:
    \begin{equation}
        \begin{aligned}
            \min_{W\in\Delta^m} &\sum_{i=1}^{m} W(i) \log \frac{W(i)}{W_{t-1}(i)} \\
            \text{s.t. } &W^T \eta_{h_s} = 0, \quad \forall s \in \{1,\dots,t\}
        \end{aligned}
    \end{equation}
\end{remark}

\paragraph{The geometric insights}
Thinking the relative entropy as a distance (which we need to be careful with, because it is not truly one)
we can see this minimization problem over a linear admissible space as a projection
on the hyperplane $\mathfrak{H}_t := \{W \text{s.t. } W^T \eta_{h_t} = 0\}$ (or for the totally corrective update, on the intersection
of $t$ hyperplanes).\newline
In~\cite{kivinen_boosting_1999}, the authors affirm that as soon as a weight distribution $W_t$ can be written
in the exponential form, i.e.~
\begin{equation}
    W_t(i) = \frac{W_{t-1}(i) e^{-\alpha_t y_i h_t(x_i)}}{Z_t(\alpha_t)}, \quad \forall i \in \{1,\dots,m\}
\end{equation}
we have for any $\alpha \in \R$:
\begin{equation}
    \argmin_{W \in \Delta^m : W^T \eta_{h_t} = 0} KL(W||W_{t-1}) = \argmin_{W \in \Delta^m: W^T \eta_{h_t} = 0} KL(W||W_t(\alpha)) = W_t
\end{equation}
That means, geometrically, that for any weight distribution $W$ in the curve the curve $\alpha \mapsto W_t(\alpha)$,
the projection of $W$ on the hyperplane $\mathfrak{H}_t$ is the only point where the curve intersects
the hyperplane.\newline
The authors finish by giving a adapted formulation of Pythagora's theorem in this setup:
Letting $W_t$ be the projection on $\mathfrak{H}_t$ of $W_{t-1}$, we have:
\begin{equation}
    KL(W_{t-1}||W_t) = KL(W_{t-1}||W_t(\alpha_t)) + KL(W_t(\alpha_t)||W_t)
\end{equation}
for any $\alpha_t \in \R$. We now truly see the meaning of $W_t(\alpha)$: it is the original update that the algorithm would like to be able to do, 
but as it does not belong to the admissible space, $W_t(\alpha)$ is projected onto the admissible space and becomes $W_t$.

\subsubsection{AdaBoost as a mirror descent (successive min-max optimization problem)}\label{sec:AdaBoost_as_a_mirror_descent}
Another, more recent, view of AdaBoost is to see it as a mirror descent algorithm.
Mirror descent is a generalization of subgradient descent, which itself is a 
generalization of gradient descent. 

\paragraph{Optimization background} 
Let $f: \X \rightarrow \R$ be a convex function, and let $\X$ be a convex subset of $\R^d$.
Suppose we want to minimize $f$ over $\X$. The problem of the gradient descent is that we 
suppose $f$ to be differentiable, which is not always the case. Therefore, we can
define the subgradient of $f$ at $x \in \X$ as a generalization of the gradient
of $f$ at $x$. 
\begin{definition}[Subgradient]
We say $g$ is a subgradient of $f$ at $x$ if for all $y \in \X$,
\begin{equation}
    f(y) \geq f(x) + \langle g, y - x \rangle
\end{equation}
\end{definition}
We denote by $\partial f(x)$ the set of subgradients of $f$ at $x$. When $f$ is 
differentiable, $\partial f(x)$ is a singleton, and $\partial f(x) = \{\nabla f(x)\}$.
When $f$ is not differentiable, $\partial f(x)$ is a convex set representing the set 
of affine functions that are below $f$ at $x$.\newline
That being said, we can now define the subgradient descent algorithm, the analog of the gradient descent 
in a non-differentiable case. Let $f$ be a convex function.
Instead of updating iteratively following the direction of the gradient of $f$, we update
iteratively following the direction of a subgradient of $f$. Eventually, we can project
the resulting point on $\X$ to ensure that we stay in $\X$.
The algorithm is described in Alg.~\ref{alg:Subgradient_descent}.

\begin{alg}[Subgradient descent]\label{alg:Subgradient_descent}
\KwData{$x_0 \in \X$}
\KwResult{$x_t$ that minimizes $f$ over $\X$}
$t \leftarrow 0$ \\
\While{\textrm{\bf{Not Stopping Criterion}} and $0 \notin \partial f(x_t)$}{
  Let $g_t \in \partial f(x_t)$ \\
  Let $x_{t+1} = \Pi_\X\left(x_t - \eta_t g_t\right)$, with $\eta_t > 0$, and $\Pi_\X$ the projection on $\X$ \\
    $t \leftarrow t+1$
  }
\Return{$x_t$}
\end{alg}

Now, the mirror descent algorithm is a generalization of the subgradient descent algorithm
that can be applied to min-max optimization problems. 
Let $f: \X \rightarrow \R$ a convex function we want to minimize. Here, we need $\X$ to be a convex subset of $\R^d$.
Without losing too much generality, we can suppose that $f$ is 
$L_f$ Lipschitz continuous (typically, this holds as soon as $\X$ is compact). Our interest lies in the cases where 
$f$ writes:
\begin{equation}
    f(x) = \max_{\lambda \in \D} \phi(x, \lambda)
\end{equation}
where $\D$ is a convex compact subset of $\R^d$, 
and $\phi: \X \times \D \rightarrow \R$ is convex in $x$ and concave in $\lambda$.
The dual function of $f$ is defined as
\begin{equation}
    f^*(\lambda) = \min_{x \in \X} \phi(x, \lambda)
\end{equation}
and we may solve the dual problem
\begin{equation}
    \min_{\lambda \in \D} f^*(\lambda)
\end{equation}
to find our solution, thanks to Von Neumann's minimax theorem~\cite{fan1953minimax}. Indeed, under our 
hypothesis, we have no duality gap, i.e.
\begin{equation}
  \exists (x^*,\lambda^*) \in \X \times \D, \forall (x,\lambda) \in \X \times \D, 
  f^*(\lambda) \leq f^*(\lambda^*) = f(x^*) \leq f(x)
\end{equation}
The way the mirror descent is a generalization of the subgradient descent is the following:
consider a differentiable $1$-strongly convex function $d: \X \rightarrow \R$.\footnote{Recall
that a function $d$ is $1$-strongly convex if for all $x,y \in \X$, $d(y) \geq d(x) + \langle \nabla d(x), y - x \rangle + \frac{1}{2}\|y - x\|^2$.}
$d$ is used to define the Bregman distance between two points $x,y \in \X$ as
\begin{equation}
    D(x,y) = d(x) - d(y) - \langle \nabla d(y), x - y \rangle
\end{equation}
The idea of the mirror descent is to succesively iterate on the primal and dual variables, with the primal update being a gradient descent step
on a slightly modified function (with the Bregman distance), and the dual update being a gradient ascent step on the dual function.
The algorithm writes as in Alg.~\ref{alg:Mirror_descent}.

\begin{alg}[Mirror descent]\label{alg:Mirror_descent}
\KwData{$x_0 \in \X$}
Let $\lambda_0 = 0$ \\
\For{$t=1,\dots,T$}{
  Compute $\tilde{\lambda}_t \in \argmax_{\lambda \in \D} \phi(x_{t-1}, \lambda)$\\
  Compute $g_t = \nabla_x \phi(x_{t-1}, \tilde{\lambda}_t)$ \\
  Choose $\alpha_t > 0$ \\
  Compute $x_t \in \argmin_{x \in \X} 
  \left\{ \alpha_t \langle g_t, x \rangle + D(x, x_{t-1}) \right\}$ \\
  $\lambda_t = \frac{\sum_{s=1}^{t-1}\alpha_s \tilde{\lambda}_s}{\sum_{s=1}^{t-1}\alpha_s}$\\
  }
\Return{$x_T$}
\end{alg}

\paragraph{AdaBoost as a mirror descent}
Here, suppose the possibilities for our classifiers are contained
in a finite set of classifiers (the base of classifiers) $\H = \{h_1,\dots,h_n\}$. This is not a big assumption,
as we will see later in Sec.~\ref{sec:AdaBoost_as_a_dynamical_system} (because AdaBoost eventually converges to a cycle).
We also suppose that given a weight distribution over the data $W \in \Delta^m$, we can find
the classifier in $\H$ that minimizes the weighted error.
\begin{theorem}[AdaBoost as a mirror descent]\label{thm:AdaBoost_as_a_mirror_descent}
    AdaBoost is a mirror descent algorithm with the following parameters:
    \begin{itemize}
        \item $\X = \Delta^m$
        \item $\mathcal{D} = \Delta^n$
        \item $\phi(W, \lambda) = W^T A \lambda$, with $A_{ij} = W(i) h_j(x_i)$
    \end{itemize}
\end{theorem} 
\begin{proof}
    Denote by $A_j$ the $j$-th column of $A$. We have, for any distribution $W$,
    the \textit{edge} of the classifier $h_j$ w.r.t.~$W$ defined as $W^T A_j$.
    For a given weight distribution $w$, the maximum edge over classifiers is thus
    \begin{equation}
        f(W) = \max_{j=1,\dots,n} W^T A_j = \max_{\lambda \in \Delta^n} W^T A \lambda = \max_{\lambda \in \Delta^n} \phi(W, \lambda)
    \end{equation}
    To ensure that the data is well classified by the final classifier no matter what 
    the true underlying data distribution is, we want to minimize the maximum edge over classifiers.
    We now are in the format of a min-max optimization problem, and we can apply the mirror descent algorithm.
    The dual function is given here by:
    \begin{equation}
        f^*(\lambda) = \min_{W \in \Delta^m} W^T A \lambda = \min_{W \in \Delta^m} \phi(W, \lambda)
    \end{equation}
    and the dual problem writes:
    \begin{equation}
        \max_{\lambda \in \Delta^n} f^*(\lambda)
    \end{equation}
    The main result from~\cite{freund_adaboost_2013} is that the sequence
    of weights and classifiers produced by AdaBoost is the sequence of primal 
    and dual variables produced by the mirror descent algorithm, setting 
    $d(W) := \sum_{i=1}^m W(i) \log W(i) + \log(m)$.

    Indeed, let's follow step by step what the mirror descent produces in this case:
    we fix $W_0 \in \Delta^m, \lambda_0 = 0$. Successively, in the loop for $t=1$ to $t=T$:
    \begin{equation}
        \begin{aligned}
            \tilde{\lambda}_t \in \argmax_{\lambda \in \Delta^n} \phi(W_{t-1}, \lambda) 
            &= \argmax_{\lambda \in \Delta^n} W_{t-1}^T A \lambda \\
            \iff A_{j_t} &= \argmax_{j=1,\dots,n} W_{t-1}^T A_{j_t} \\
        \end{aligned}
    \end{equation}
    where $j_t$ is the index of the classifier that minimizes the weighted error at iteration $t$.
    It is necessarily existing because the problem is linear on the simplex, so the optimum is 
    reached at one of the vertices. This step corresponds exactly to fitting the classifier
    $h_t$ to the data, i.e.~to find the classifier that minimizes the weighted error, which
    is precisely the first phase of AdaBoost.
    Then, as $A_{j_t} = \argmax_{j=1,\dots,n} W_{t-1}^T A_{j_t}
    \iff A_{j_t} \in \partial f(W_{t-1})$, we have that:
    \begin{equation}
        \begin{aligned}
            g_t &= \nabla_{W} \phi(W_{t-1}, \tilde{\lambda}_t) = A_{j_t} \\
            \iff g_t &= \nabla_{W} \max_{\lambda \in \Delta^n} \phi(W_{t-1}, \lambda) \\
            \iff g_t &= \nabla_{W} f(W_{t-1}) \\
        \end{aligned}
    \end{equation}
    The next step in the mirror descent is to choose $\alpha_t > 0$ and to compute
    \begin{equation}
        \begin{aligned}
            W_t &\in \argmin_{W \in \Delta^m} \left\{ \alpha_t \langle g_t, W \rangle + D(W, W_{t-1}) \right\} \\
            &\in \argmin_{W \in \Delta^m} \left\{ \alpha_t W^T A_{j_t} + \sum_{i=1}^m W(i) \log \frac{W(i)}{W_{t-1}(i)} + \sum_{i=1}^{m}\left(W_{t-1}(i) - W(i)\right)\right\} \\
        \end{aligned}
    \end{equation}
    The first order optimality condition on each component of $w$ implies that:
    \begin{equation}
        \begin{aligned}
            \forall i, &\frac{d}{dW(i)}\left(W(i)\left(\alpha_t y_i h_{j_t}(x_i) + \log \frac{W(i)}{W_{t-1}(i)} + W_{t-1}(i) - W(i)\right)\right) = 0 \\
            &\iff \forall i, \alpha_t y_i h_{j_t}(x_i) + \log \frac{W(i)}{W_{t-1}(i)} = 0 \\
            &\iff W(i) = W_{t-1}(i) \exp(-\alpha_t h_{j_t}(x_i)y_i)
        \end{aligned}
    \end{equation}
    which is exactly the weight update from original AdaBoost. The last normalization step is 
    straightforward.
\end{proof}
The fact that AdaBoost iterations are actually exactly the same as a mirror descent
for $d$ fixed earlier allows the authors to prove new properties and complexity bounds 
on the algorithm~\cite{freund_adaboost_2013}.

\subsection{AdaBoost as traditional machine learning methods}\label{sec:traditional_machine_learning_views}
\subsubsection{Why does AdaBoost generalize that well?}\label{sec:AdaBoost_generalizes_well}
Recently, there was a gain of interest about what are called \textit{interpolating classifiers}. It has been proven
that they generalize very well~\cite{liang2021interpolating}. In 2017, Wyner et al.~\cite{wyner2017explaining}
introduced the term of \textit{interpolating} classifiers, and they have shown that we 
can see Random Forest and AdaBoost classifiers as interpolating classifiers. This is a part of the explanation
of why AdaBoost generalizes so well. Another part of the explanation is the links between AdaBoost and
the margin theory, as explained in Sec.~\ref{sec:AdaBoost_as_a_regularized_path_to_a_maximum_margin_classifier}.

\paragraph{Interpolating classifiers}
\begin{definition}[Interpolating classifier]
A classifier $h$ is said to be an \textit{interpolating} classifier if for each sample of 
the training set, the classifier assigns the correct label, i.e.~$\forall i \in \{1,\dots,m\},
f(x_i) = y_i$. 
\end{definition}
\begin{remark}
    The term \textit{interpolating} can be explained geometrically by thinking
    of a set of points we want to interpolate with a class of smooth functions, let's say polynomial
    for instance. Most people tend to say that such classifier will overfit, achieving $100\%$ accuracy
    on the training set, and having a poor generalization error (e.g.~one nearest-neighbor).
\end{remark}
However, Random Forests and AdaBoost are example of interpolating classifier that still seem
to generalize well.

\paragraph{AdaBoost as an interpolating classifier \& the double descent}
Experimentally, in~\cite{wyner2017explaining}, it is shown that AdaBoost tends to smooth the
decision boundary the more iterations it does. In simple words, it may need $T$ iterations for 
AdaBoost to reach $100\%$ accuracy on the training set, but after $T$ iterations, the generalization
error may be very poor. However, continuing the iterations will eventually reduce the generalization error 
as the resulting classifier will act as an average of several interpolating classifiers (i.e.~classifiers 
that reach $100\%$ accuracy on the train set). In some way, this can be linked with the deep learning double descent phenomena (see~\cite{nakkiran2021deep}).
Indeed, the first descent phase consists in fitting the training data, and the second phase consists
in letting the algorithm run more iterations to reduce the generalization error.\newline

To put it in a more formal way, suppose we need $T$ iterations to reach $100\%$ accuracy on the training set.
Then, the resulting classifier $H_T=\sum_{t=1}^{T}\alpha_t h_t$ is an interpolating classifier. We suppose we ran $KT$ iterations
of AdaBoost. It is reasonable to think that all following classifiers are interpolating classifiers as well:
\begin{equation}
    H_T^k = \sum_{t=1}^{T}\alpha_{kT + t} h_{kT + t}, \quad \forall k \in \{1,\dots,K-1\}
\end{equation}
We then can see the resulting classifier $H_{KT}$ as an average of interpolating classifiers $H_T^k$:
\begin{equation}
    H_{KT} = \sum_{k=1}^{K} H_T^k = \sum_{t=1}^{KT} \alpha_t h_t
\end{equation}
Experimentally (see~\cite{wyner2017explaining}), $H_{KT}$ seems to generalize way better than any of the $H_T^k$. This makes us 
think that AdaBoost can be seen as an average of interpolating classifiers. In~\cite{wyner2017explaining},
this property is called \textit{self-averaging property of boosting}.\newline

\begin{remark}
    This behavior of boosting algorithms actually have been identified more than $20$ years ago, by 
    Schapire et al.~\cite{schapire2003boosting}. They identified that the generalization error would 
    likely decrease and AdaBoost would not overfit as it would be predicted by the VC-dimension theory. A few
    theoritical arguments have been proposed to explain this phenomena. The authors from~\cite{bartlett1998boosting} show that 
    the generalization error of the resulting classifier $H_T$ can be upper bounded. Indeed, let $d$ be the
    VC-dimension of $\H$, let $\mathcal{D}$ be the underlying distribution of the data, and let $\mathcal{S}$ be
    the training set of size $m$ (i.e.~the (discrete) distribution of the data we have observed). Then, we have:
    \begin{equation}
        \P_\mathcal{D}\left(yH_T(x) \leq 0\right) \leq \P_\mathcal{S}\left(yH_T(x) \leq \theta\right) + O\left(\sqrt{\frac{d}{m\theta^2}}\right)
    \end{equation}
    for any $\theta > 0$ with high probability. Note that this bound is independant of the number of 
    iterations $T$ of AdaBoost. It can partially explain that boosting algorithms do not overfit. 
\end{remark}
However, quantitatively, the bounds are weak and do not explain the \textit{double descent phenomena}\footnote{Note that this
term is used to echo the phenomena we all know about in deep learning from Nakkiran et al.~\cite{nakkiran2021deep}. However, this term is slightly
unadapted here. Indeed, there is no \textit{double} descent as there is no ascent. We only see a phenomena where the generalization error
continues to decrease after achieving perfect classification.}. However, an older view of AdaBoost also partially explains this phenomena, 
by showing that AdaBoost increases the $l_1$ margin of the model over the iterations (see Sec.~\ref{sec:AdaBoost_as_a_regularized_path_to_a_maximum_margin_classifier}).

\paragraph{AdaBoost as a regularized path to a maximum margin classifier}\label{sec:AdaBoost_as_a_regularized_path_to_a_maximum_margin_classifier}
This idea has been proposed in 2004 by Rosset et al.~\cite{rosset_boosting_2004}, but it bases 
itself on results that were directly underlined by Schapire et al.~\cite{bartlett1998boosting}. This
formulation comes from the combinations of two points of view: the first one is the Gradient
descent formulation explained in Sec.~\ref{sec:AdaBoost_as_a_gradient_descent_algorithm}, and the second one concentrates
on the effects of boosting on the margin $y_i H_t(x_i)$ of the classifier $H_t$. It is probably one of the most useful to
understand why AdaBoost converges and why it generalizes well.\newline
Recall the two different loss functions $c(\gamma)$ presented in Sec.~\ref{sec:AdaBoost_as_a_gradient_descent_algorithm}:
\begin{equation}
    \begin{aligned}
        \text{Exponential loss:} \quad c(\gamma) &= e^{-\gamma} \\
        \text{Binomial log-likelihood loss:} \quad c(\gamma) &= \log(1+e^{-\gamma})
    \end{aligned}
\end{equation}
Those two loss functions are actually very similar, because if $\eta_{H_T} \geq 0$, 
the two functions behave as exponential loss~\cite{rosset_boosting_2004}.\newline
In parallel, consider the additive model $H_T(x) = \sum_{t=1}^{T} \alpha_t h_t(x)$.
We can link this additive model to SVM theory by considering the margin of the classifier
w.r.t.~observation $x_i$:
\begin{equation}
    y_i H_T(x_i) = \sum_{t=1}^{T} \alpha_t y_i h_t(x_i)
\end{equation}
Suppose our classifier achieves $100\%$ accuracy on the training set, i.e.~$y_i H_T(x_i) \geq 0$ for all $i \in \{1,\dots,m\}$.
Then, we can define the margin of the classifier as the minimum margin of the observations w.r.t.~a certain distance (e.g.~$l_p$ distance):
\begin{equation}
    m_p(H_T) = \min_{i \in \{1,\dots,m\}} \frac{y_i H_T(x_i)}{||(\alpha_1,\dots,\alpha_T)||_p}
\end{equation}
Now the interesting part is that there is a link between AdaBoost and the $l_1$ margin.
In case of separable training data, AdaBoost produces a non-decreasing sequence w.r.t. the $l_1$ margin
of the model:
\begin{equation}
    \forall t, m_1(H_{t+1}) \geq m_1(H_t)
\end{equation}
What is important here is for the generalization error. Indeed, if we incease the 
$l_1$ margin of the model, as for SVM, we are likely to decrease the generalization error.\newline
This is to put in parallel with the double descent phenomena explained in Sec.~\ref{sec:AdaBoost_generalizes_well}.
Indeed, these two vision of AdaBoost are very similar. The double descent phenomena can be explained here 
by the regularization effect of AdaBoost which increases the $l_1$ margin of the model.\newline

\begin{figure}
    \centering
    \begin{tikzpicture}[scale=2,
        thick,
        >=stealth',
        dot/.style = {
          draw,
          fill = white,
          circle,
          inner sep = 0pt,
          minimum size = 4pt
        },
        cross/.style={cross out, draw=black, minimum size=2*(#1-\pgflinewidth), inner sep=0pt, outer sep=0pt}
      ]
      \coordinate (O) at (0,0);
        \draw[line width=0.5mm, blue] plot coordinates {(0,0) (1,1) (2,0) (3,1) (4,0)};
        \draw[dashdotted, red] plot[smooth] coordinates {(0,0.1) (1,1.2) (2,0.1) (3,1.2) (4,0.1)};
        \draw[dashdotted, red] plot[smooth] coordinates {(0,0.2) (1,1.4) (2,0.2) (3,1.4) (4,0.2)};
        \draw[dashdotted, red] plot[smooth] coordinates {(0,0.3) (1,1.6) (2,0.3) (3,1.6) (4,0.3)};
        \draw[dashdotted, red] plot[smooth] coordinates {(0,-0.1) (1,0.8) (2,-0.1) (3,0.8) (4,-0.1)};
        \draw[dashdotted, red] plot[smooth] coordinates {(0,-0.2) (1,0.6) (2,-0.2) (3,0.6) (4,-0.2)};
        \draw[dashdotted, red] plot[smooth] coordinates {(0,-0.3) (1,0.4) (2,-0.3) (3,0.4) (4,-0.3)};
        \draw[help lines, dashdotted, xstep=1cm, ystep=1cm] (0,-1) grid (4,2);

        \node at (0.1,-0.8)  {\Cross};
        \node at (0.9,0.5)  {\Cross};
        \node at (3.1,0.5)  {\Cross};
        \node at (2,-0.5)  {\Cross};
        \node at (3.7,-0.3)  {\Cross};

        \node at (0.1,0.9)  {\Square};
        \node at (0.9,1.4)  {\Square};
        \node at (3.1,1.4)  {\Square};
        \node at (2,0.5)  {\Square};
        \node at (3.7,0.8)  {\Square};

        \draw[->] (-0.1,0.4) arc (-30:30:-0.3);
        \draw[->] (-0.1,-0.4) arc (30:-30:-0.3);

        \draw (-0.1, 0.2) node[left] {$T\to\infty$};
        \draw (-0.1, -0.2) node[left] {$T\to\infty$};

    \end{tikzpicture}
    \caption{Illustration of the increase of the $l_1$ margin of the model over the iterations of AdaBoost. The decision boundaries of the iterations of AdaBoost are 
    shown in red. The blue line represents the decision boundary of the maximum margin classifier.
    Eventually, AdaBoost will converge to the maximum margin classifier.}
    \label{fig:maximum_margin}
    \end{figure}
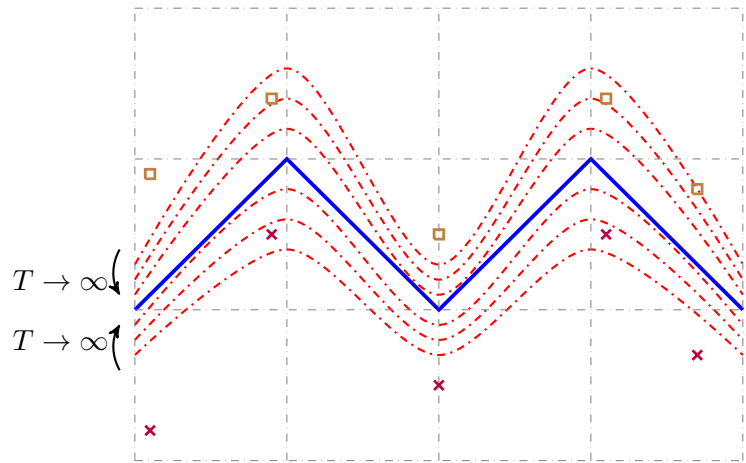

\subsubsection{AdaBoost as a Kernel method - Boosting for regression}\label{sec:AdaBoost_as_a_Kernel_method}
Very recently, in 2019, Aravkin et al.~\cite{aravkin2019boosting} proposed a new view of boosting methods.
This view is more adapted for regression problems, that is why we will adapt our notations from before.
However, the authors affirm this view can be adapted to classification problems as well.

\paragraph{General algorithm}
The way to see boosting for regression is actually very close from the way we see boosting as successive optimization problems.
Indeed, given data $y_1,\dots,y_m \in \R$, we want to find a function $H: \X \rightarrow \R$ that minimizes the regression error
\begin{equation}
    H = \argmin_{H} L(H(X) - y)
\end{equation}
where $L$ is a loss function and $X={(x_1,\dots,x_m)}^T$. As often for regression problem, we want to prevent overfitting by adding
a regularization term to the loss function. 
\begin{definition}[Optimization problem in boosting for regression]
    Choosing a kernel matrix $K \in \R^{m \times m}$, we can write
    our optimization problem as
    \begin{equation}
        H = \argmin_{H} L(H(X) - y) + \gamma X^T K X =: J(X,y)
    \end{equation}
    where $\gamma > 0$ is a regularization parameter.
\end{definition}
Then, we can see boosting as stated in Algorithm~\ref{alg:AdaBoost_regression}.

\begin{alg}[Boosting for regression]\label{alg:AdaBoost_regression}
Set $h_0 = \argmin_{H} J(X,y)$. \\
\For{$t=1,\dots,T$}{
  Update $y_t = y - H_{t-1}(X)$ \\ \CommentSty{$y_t$ is the residual error at iteration $t$.
  We can see it as the new target for the next iteration. That will boost the importance of the examples we 
  did not approximate well at the previous iteration, and will reduce the importance of the examples we approximated correctly.} \\
  Compute $h_t = \argmin_{H} J(X,y_t)$ \\
  Compute $H_t = H_{t-1} + h_t$ \\
  }
\Return{$H_T$}
\end{alg}

\paragraph{Link with Kernel methods for linear regression}
What is the link with Kernel methods? Fix $L(H(X) - y) = ||y - H(X)||^2$ where $||.||$ is the Euclidean norm.
Suppose as well that $\H$ is the set of linear functions, i.e.~$\H = \{H: X \mapsto \beta X, \beta \in \R^{m\times m}\}$,
and $y \sim \mathcal{N}(0, \sigma^2 I_m)$. Let $\lambda$ the kernel scale parameter related to $K$.
Then, we can identify $h_t$ to $\beta_t$, and we can write, setting $\gamma = \frac{\sigma^2}{\lambda}$:
\begin{equation}
    \begin{aligned}
    \beta_t &= \argmin_{\beta\in\R^m} ||y_t - \beta X||^2 + \gamma X^T K X \\
    &= \lambda K \beta^T{(\lambda \beta K \beta^T + \sigma^2I_m)}^{-1} y_t \\
    \end{aligned}
\end{equation}
which has an explicit form. Setting $P_\lambda = \lambda \beta K \beta^T$, and 
$h = \beta X$, the predicted data output writes:
\begin{equation}
    \begin{aligned}
    h_t(X) &= \argmin_{h\in\R^m} ||y_t - h||^2 + \sigma^2 h^T P_\lambda^{-1}h\\
    &= P_\lambda{(P_\lambda+\sigma^2 I)}^{-1}y_t \\
    \end{aligned}
\end{equation}
\begin{definition}[Boosting kernel]
    For all $t\geq 1$, we call \textbf{boosting kernel} the quantity $P_{\lambda,t}$, defined by:
    \begin{equation}
        P_{\lambda,t} = \sigma^2 \left(I - P_\lambda{(P_\lambda+\sigma^2 I)}^{-1}\right)^{-t} - \sigma^2 I
    \end{equation}
\end{definition}
\begin{proposition}[$H_t$ is a kernel-based estimator]
One can show~\cite{aravkin2019boosting} that the updated classifier at each iteration
$H_t$ is a kernel-based estimator where the kernel is $P_{\lambda,t}$. 
\end{proposition}
\begin{proof}
    \textit{We will here only give a sketch of the proof.}\newline
    According to the boosting scheme defined for regression: set $S_\lambda = P_\lambda{(P_\lambda+\sigma^2 I)}^{-1}$ to
    simplify the notations. Then, we have:
    \begin{equation}
        \begin{aligned}
            H_0(X) &= S_\lambda y \\
            H_1(X) &= S_\lambda y + S_\lambda (I - S_\lambda)y \\
            &\vdots \\
            H_t(X) &= S_\lambda \sum_{i=0}^{t-1} {(I - S_\lambda)}^i y \\
        \end{aligned}
    \end{equation}
    However, we also have that $S_{\lambda,t} := P_{\lambda,t}{(P_{\lambda,t}+\sigma^2 I)}^{-1}$
    simplifies to $S_\lambda \sum_{i=0}^{t-1} {(I - S_\lambda)}^i$ for all $t \geq 1$. Therefore,
    at iteration $t$, boosting returns the same estimator as the kernel-based estimator with kernel
    $P_{\lambda,t}$. 
\end{proof}
The authors provide more detailed proof and further insights of boosting as
a kernel method in~\cite{aravkin2019boosting}.

\subsubsection{AdaBoost as a Product of Experts}\label{sec:AdaBoost_as_a_product_of_experts}
This subsection is based on~\cite{edakunni2012boosting}.
\paragraph{Product of Experts models}
The idea behind Product of Experts (PoE) models is that we have access to several experts models, and
we combine their predictions to get a better prediction. This is a very general idea, the same
that motivates every ensemble method. PoE differ from other ensemble methods in the way that we 
suppose we already have access to the experts models and look for a way to combine the models,
rather than generating them in the best way. It is a more probabilistic approach to ensemble methods.\newline
According to~\cite{edakunni2012boosting}, boosting can be seen as incremental learning in PoE.
Incremental learning in PoE is an algorithm close to boosting that generates an estimator iteratively
using PoE.

\begin{alg}[Incremental learning in PoE]\label{alg:Incremental_learning_PoE}
\KwData{$(x_1, y_1),\dots, (x_m, y_m) \in \X \times \Y$}
Initialize $W_0 = (1/m,\dots,1/m)$ \\
\For{$t=1,\dots,T$}{
    Find a hypothesis $h_t$ such that $\sum_{i=1}^{m} \frac{W_j(i)}{\P(y_i|x_i,h_{j-1})} \leq 2$ (for binary classification) \\
    Update $W_t(i) = W_{t-1}(i)(1-\P(y_i|x_i,h_t))$ for $i = 1,\dots,m$ \\
    Normalize $W_t$ \\
    }
\Return $H(x) = \mathrm{sign}\left(\P(y=1|x,h_1,\dots,h_T) - \P(y=-1|x,h_1,\dots,h_T)\right)$
\end{alg}

\paragraph{Boosting as incremental learning in PoE}
\begin{theorem}[Boosting as incremental learning in PoE]
    The AdaBoost algorithm is equivalent to incremental learning in PoE.
\end{theorem}
\begin{proof}
    Consider the estimator $h_t$ at iteration $t$ of AdaBoost for a binary classification problem. 
    We can write the probability of classifying correctly an example $x_i$ at iteration $t$ as
    \begin{equation}
        \begin{aligned}
        \P(y_i=y|x_i,h_t) = &\P(y_i=y|h_t(x_i)=y)\P(h_t(x_i)=y|x_i) \\
                                + &\P(y_i=y|h_t(x_i)=-y)\P(h_t(x_i)=-y|x_i)
        \end{aligned}
    \end{equation}
    Suppose that the error from classifier $h_t$ is symmetric, i.e.
    \begin{equation}
        \P(y_i=y|h_t(x_i)=-y) = \P(y_i=-y|h_t(x_i)=y) = P_t
    \end{equation}
    Then, we have
    \begin{equation}
        \P(y_i=y|x_i,h_t) = (1 - \P(h_t(x_i) \ne y_i|x_i))\left(1 - P_t\right) + \P(h_t(x_i) \ne y_i|x_i)P_t
    \end{equation}
    Now, on the one hand, we have
    \begin{equation}
        \begin{aligned}
            \sum_{i=1}^{m} \frac{W_{j-1}(i)}{\P(y_i|x_i,h_t)} &= \sum_{i=1}^{m} \frac{W_{t-1}(i)}{(1 - \P(h_t(x_i) \ne y_i|x_i))\left(1 - P_t\right) + \P(h_t(x_i) \ne y_i|x_i)P_t} \\
            &= \sum_{h_t(x_i) = y_i} \frac{W_{t-1}(i)}{\left(1 - P_t\right)} +  \sum_{h_t(x_i) \ne y_i} \frac{W_{t-1}(i)}{P_t}\\
        \end{aligned}
    \end{equation}
    by supposing our classifier satisfies $\P(h_t(x_i) = 1|x_i) \in \{0,1\}$, which means that it produces
    decisions which have no random component (which is almost always the case for AdaBoost, as we use decision trees). 
    On the other hand, we must impose the condition $\sum_{i=1}^{m} \frac{W_{t-1}(i)}{\P(y_i|x_i,h_t)} \leq 2$, which implies:
    \begin{equation}
        \epsilon_t := \sum_{h_t(x_i) \ne y_i} W_{t-1}(i) \leq P_t \leq \frac{1}{2}
    \end{equation}
    This hypothesis is reasonable since we hope our classifier to perform better than random guessing.
    Setting $P_t = \epsilon_t$, we have
    \begin{equation}
        P_t = \frac{e^{-\alpha_t}}{e^{-\alpha_t}+e^{\alpha_t}}
    \end{equation}
    We finally need to update the weights, which is made in Alg.~\ref{alg:Incremental_learning_PoE} by
    \begin{equation}
        W_t(i) = W_{t-1}(i)(1-\P(y_i|x_i,h_t))
    \end{equation}
    and in our case, 
    \begin{equation}
        \begin{aligned}
        1-\P(y_i|x_i,h_t) &= \begin{cases}
                P_t & \mbox{if } h_t(x_i) = y_i \\
                1 - P_t & \mbox{otherwise.}
            \end{cases} \\
        &= \frac{e^{-\alpha_t y_i h_t(x_i)}}{e^{-\alpha_t}+e^{\alpha_t}} \\
        \end{aligned}
    \end{equation}
    Therefore we recover the weight update from AdaBoost.
\end{proof}
We can now write the AdaBoost algorithm as incremental learning in PoE (see Alg.~\ref{alg:AdaBoost_PoE}).

\begin{alg}[AdaBoost as incremental learning in PoE]\label{alg:AdaBoost_PoE}
\KwData{$(x_1, y_1),\dots, (x_m, y_m) \in \X \times \{-1, 1\}$}
Initialize $W_0 = (1/m,\dots,1/m)$ \\
\For{$t=1,\dots,T$}{
    Set $\epsilon_t = \sum_{h_t(x_i) \ne y_i} W_{t-1}(i)$ \\
    Find a hypothesis $h_t$ such that $\epsilon_t \leq \frac{1}{2}$ that minimizes $\epsilon_t$ \\ 
    Set $\alpha_t = \frac{1}{2} \log \frac{1 - \epsilon_t}{\epsilon_t}$ \\
    $\P(y_i|x_i,h_t) = 1 - \frac{e^{-\alpha_t y_i h_t(x_i)}}{e^{-\alpha_t}+e^{\alpha_t}}$ \\
    Update $W_t(i) = W_{t-1}(i)(1-\P(y_i|x_i,h_t))$ for $i = 1,\dots,m$ \\
    Normalize $W_t$ \\
    }
\Return $H(x) = \mathrm{sign}\left(\sum_{t=1}^{T}\alpha_t h_t(x)\right)$
\end{alg}
In practice, in~\cite{edakunni2012boosting}, it is shown that Alg.~\ref{alg:AdaBoost_PoE} 
can be derived to obtain significantly better results than the original AdaBoost algorithm. 

\subsubsection{AdaBoost as a dynamical system: experiments and theoritical insights}\label{sec:AdaBoost_as_a_dynamical_system}

\paragraph{Diversity and cycling behavior}\label{subsec:Diversity}
Some may see boosting algorithm as a way to increase diversity of a set
of weak learners to produce a better additive model. We will discuss here how
diversity intervene in boosting algorithms, how it can be measured but 
also how it is limited. Indeed, by seeking for an estimator which must be 
very diverse from the previous one, AdaBoost often falls into a cycling behavior
which is not the purpose for which it was designed.\newline

\paragraph{What is diversity, and how do we measure it?}
The key for boosting algorithms to work is to be able to combine some different
weak learners. Indeed, if the set of weak learners in which we choose our $h_t$
is not diverse enough, we may not be able to increase too much the accuracy from
a single weak learner to an additive model of weak learners.\newline
Thus, it is important that the weak learners iteratively chosen by AdaBoost are 
different enough from each other. This is what we call \textit{diversity}.\newline
There are several ways to measure diversity, as there are several ways to measure the
efficiency of an estimator~\cite{grandini2020metrics} (precision, recall, F1-score, etc.). 
Let's fix a definition here. We will be using the definition of diversity from~\cite{li2012diversity}:\newline
\begin{definition}[Diversity of a set of weak learners] \label{def:diversity}
    Given a set of weak learners $\H = \{h_1,\dots,h_T\}$, we define the diversity
    of $\H$ as
    \begin{equation}
        \mathrm{div}(\H) = 1 - \frac{2}{T(T+1)} \sum_{1\leq t\ne s\leq T} \mathrm{sim}(h_t, h_s)
    \end{equation}
    where $\mathrm{sim}(h_t, h_s)$ is a measure of the similarity between $h_t$ and $h_s$ defined as
    \begin{equation}
        \mathrm{sim}(h_t, h_s) = \frac{1}{m} \sum_{i=1}^{m} h_t(x_i) h_s(x_i)
    \end{equation}
\end{definition}
We can see that $\mathrm{sim}(h_t, h_s) \in [-1,1]$, and that $\mathrm{sim}(h_t, h_s) = 1$ if and only if
$h_t = h_s$, and $\mathrm{sim}(h_t, h_s) = -1$ if and only if $h_t = -h_s$.\newline
Therefore, the diversity of $\H$ increases if we add very different weak learners to $\H$
than the ones already in $\H$, and decreases if we add very similar weak learners to $\H$.\newline

Thus, for AdaBoost as for any other ensemble method, the key for the algorithm to work is to have a
diverse enough set of weak learners (e.g.~decision trees). 

However, as we will illustrate in the next subsubsection, even with a diverse enough set of estimators,
AdaBoost iterations do not necessarily produces a subset of diverse estimators.

\paragraph{Diversity is not always the key: \textit{Does AdaBoost always cycle?}}
In~\cite{rudin2004dynamics}, the authors show that for some problems, AdaBoost iterations
become cyclic. This means that the new weak learners that we are adding to our ensemble model
decrease diversity of the set (in the sense of Def.~\ref{def:diversity}).\newline
To illustrate this, we can consider the following example. Let $\X = [0,1] \times [0,1]$ and
$\Y = \{-1,1\}$. For each $x$, we assign a label $y$ as follows:
\begin{equation}\label{eq:toy_problem}
    y(x) = \left\{
        \begin{array}{ll}
            1 & \mbox{if } x_1 \leq \frac{1}{4} \text{ or } x_2 \leq \frac{1}{4} \text{ or } x_2 \geq \frac{3}{4} \\
            -1 & \mbox{otherwise.}
        \end{array}
    \right\}
\end{equation}
Consider $\H$ the set of all decision stumps on $\X$. $\H$ is a diverse set of weak learners.
However, if we run AdaBoost on this problem, we will see that the algorithm will produce a cyclic sequence
very quickly. Infact, we will likely get a sequence of $3$ different decision stumps that are repeated over and over again.
\begin{figure}[H]
    \centering
    \includegraphics[width=\textwidth]{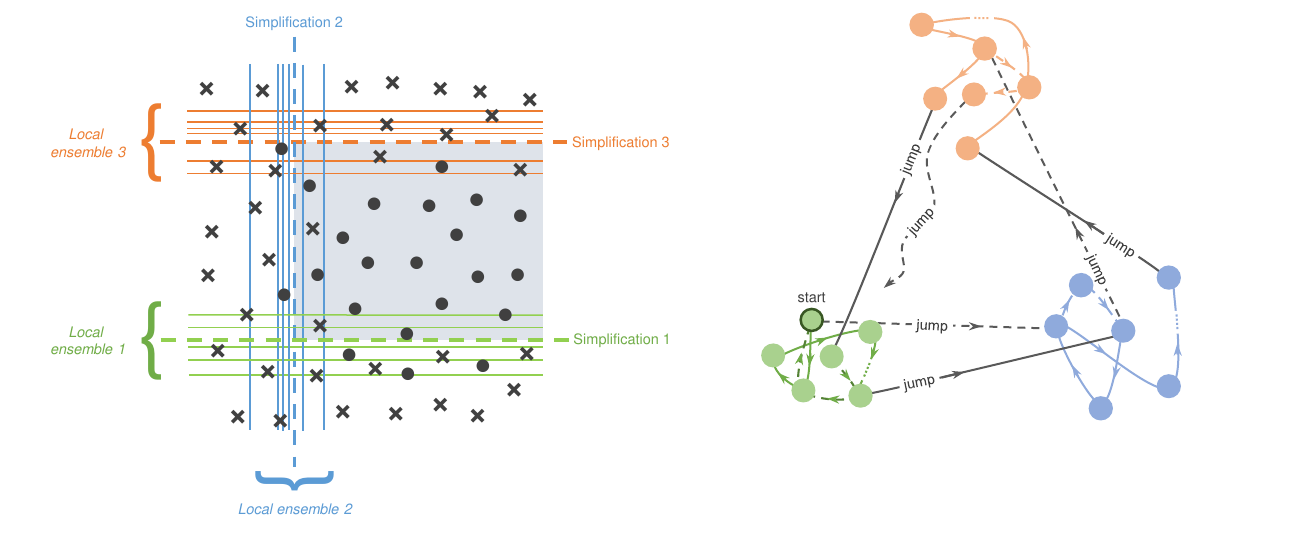}
    \caption{Example of cyclic AdaBoost iterations. This represents the toy problem defined in
    Eq.~\ref{eq:toy_problem}. The image row shows the decision stumps chosen by AdaBoost,
    the second image illustrate them in the space of decision stumps where the distance between the
    classifiers can be seen as the opposite of the similarity defined in Sec.~\ref{subsec:Diversity}.}
    \label{fig:AdaBoost_diversity}
\end{figure}

This is illustrated in Fig.~\ref{fig:AdaBoost_diversity}. More difficult problems don't necessarily
enlighten this phenomena this quickly, but we can still observe it in higher dimensions real-world problems.\newline
The behavior that we tend to see is that in the first place, AdaBoost tries to cover a \textit{base space} of 
estimators, and then, it seems to try to find the best combination of the estimators in this base space
by repeating the same estimators over and over again but not necessarily as often as the others.\newline

This intution has been formalized theoritically in the open problem presented in~\cite{rudin2012open} in 2012, and this
conjecture has very recently been proven in~\cite{belanich2012convergence} in 2023. We will present the main ideas in
the next subsubsection.

\paragraph{Theoritical insights}\label{subsec:theoritical_insights}
This subsubsection is based on~\cite{belanich2012convergence}. This article proposes an original
way to see AdaBoost. It is very formal, and we will try to keep most of their notations. Also,
we volontary simplify the work they have done, which is way more riguourous and precise than this subsection.
This subsection only stands to give intuitions of the results they have proven.\newline
A dynamical system can be defined as follows:
\begin{definition}[Dynamical system]
    A dynamical system is a tuple $(\X,\Sigma,\mu,f)$ where:
    \begin{itemize}
        \item $\X$ is a compact metric space
        \item $\Sigma$ is a $\sigma$-algebra on $\X$
        \item $\mu$ is a probability measure on $\X$
        \item $f$ is a measurable function from $\X$ to $\X$ which describes the evolution of the system
    \end{itemize}
\end{definition}
We will consider AdaBoost as a dynamical system over the weights of the data at each iteration.
Thus, in our case, the dynamical system we consider is $(\Delta_m, \Sigma_m, \mu, \A)$ where:
\begin{itemize}
    \item $\Delta_m$ is the $m$-dimensional simplex, i.e.~$\Delta_m = \{p \in \mathbb{R}^m | \sum_{i=1}^{m} p_i = 1, p_i \geq 0\}$
    \item $\Sigma_m$ is the Borel $\sigma$-algebra on $\Delta_m$
    \item $\mu$ is a measure we will define later
    \item $\A$ is the AdaBoost update over the weights
\end{itemize}
The main goal of the paper is to be able to apply the \textit{Birkhoff Ergodic Theorem}~\ref{thm:Birkhoff}~\cite{birkhoff1931proof} to the previous dynamical system.
\begin{theorem}[Birkhoff Ergodic Theorem] \label{thm:Birkhoff}
    Let $(\X,\Sigma,\mu,\A)$ be a dynamical system and $f$ be a measurable function from $\X$ to $\mathbb{R}$.
    Then, there exists a measurable function $f^*$ such that:
    \begin{equation}
        \lim_{n \to \infty} \frac{1}{T} \sum_{i=1}^{T} f(\A^i(x)) = f^*(x)
    \end{equation}
    for $\mu$-almost every $x \in \X$. Also, we have that $f^*$ is $\A$-invariant, i.e.~$f^* \circ \A = f^*$.
\end{theorem}
Applying this theorem to our dynamical system would allow us to prove the convergence of AdaBoost (w.r.t.~the weights), and to prove that there exists a fixed
point of the system.\newline
To apply such theorem, we need to ensure that there exists a measure $\mu$ such that $\A$ is $\mu$-preserving. This is what
the \textit{Krylov-Bogoliubov Theorem}~\ref{thm:Krylov-Bogoliubov} states.
\begin{theorem}[Krylov-Bogoliubov Theorem] \label{thm:Krylov-Bogoliubov}
    If $(W,N)$ is a metric compact space and $g: W \to W$ is a continuous function, then there exists a Borel probability measure $\mu$ on $W$
    such that $g$ is $\mu$-preserving.
\end{theorem}
We now see that in order to apply the Birkhoff Ergodic Theorem, we only have to apply the Krylov-Bogoliubov Theorem to our (well-chosen) dynamical system.
What we first can observe is that it is only relevant to set our first weight $w_1$ in the interior subset of the simplex 
$\Delta_m^\circ = \{w \in \mathbb{R}^m | \sum_{i=1}^{m} w_i = 1, w_i > 0\}$. Indeed, if we set $w_1$ to be $0$ for some $i$, then the weight will stay $0$ for all the iterations.\newline
Let $\eta$ a dichotomy of the data, i.e.~$\eta = (y_1 h(x_1),\dots,y_m h(x_m)) \in {\{-1,1\}}^m$ which is positive at the $i$th compenent only if $h$ classifies correctly sample $x_i$. 
We can define a set of weights $\pi(\eta) := \{w\in \Delta_m, \eta \in \argmin_{\eta'} \eta'^T w\}$ the set of weights such that the estimator $h$ that minimizes the 
weighted error $w^T \eta_h$ verifies $\eta_h = \eta$.\newline
Define also $\pi^+(\eta) = \{w \in \pi(\eta), \eta^T w > 0\}$ the set of weights in $\pi(\eta)$ that are make non-zero error on mistake dichotomy $\eta$. Then,
we can define $\Delta_m^+ = \cup_{\eta \in {\{-1,1\}}^m} \pi^+(\eta)$ the set of weights that make non-zero error on at least one mistake dichotomy.\newline
The first big result that is established in the paper is the following:
\begin{proposition}[Continuity of AdaBoost Update]
    The AdaBoost update $\A$ is continuous over the set $\cup_{\eta \in {\{-1,1\}}^m} \pi^\circ(\eta)$. 
\end{proposition} 
This is a key result to apply the Krylov-Bogoliubov Theorem. 
\begin{proof}
Indeed, setting $w^s \to w$ a sequence of weights in $\cup_{\eta \in {\{-1,1\}}^m} \pi^\circ(\eta)$, we have that:
\begin{equation}
    \begin{aligned}
        \A(w^s)(i) &= \frac{1}{2} w^s(i) \left(\frac{1}{\eta w^s}\right)^{\eta(i)} \left(\frac{1}{1 - \eta w^s}\right)^{1-\eta(i)} \\
        &\underset{s\to\infty}{\to} \frac{1}{2} w(i) \left(\frac{1}{\eta w}\right)^{\eta(i)} \left(\frac{1}{1 - \eta w}\right)^{1-\eta(i)} = \A(w)(i)
    \end{aligned}
\end{equation}
because $w^s \to w$ and $\eta w^s \to \eta w$. So we have that $\A(w^s) \to \A(w)$, which proves the continuity of $\A$ over $\cup_{\eta \in \{-1,1\}^m} \pi^\circ(\eta)$.
\end{proof}
Their second main result is that the relative error of each weak classifier produced by AdaBoost can be lower bounded.
\begin{proposition}[Lower bound on the relative error of weak classifiers]
    \begin{equation}
        \forall t \in \N, w_t^T \eta_{h_t} \geq \frac{1}{2^{t+1}}
    \end{equation}
\end{proposition}
The proof is by induction, but is slightly more technical and needs to introduce more notations that we won't detail here.\newline
In~\cite{belanich2012convergence}, the authors don't apply the Birkhoff Ergodic Theorem to the whole set $\Delta_m^+$, but only to a subset of it: the limit set of AdaBoost
that can be reached by an infinite number of iterations starting from weights in $\Delta_m^+$ that they denote $\Omega_\infty^+ = \cap_{t=1}^{\infty} \A^t(\Delta_m^+)$.
\begin{proposition}[Compactness of AdaBoost limit set]
    $\Omega_\infty^+$ is compact.
\end{proposition}
The compactness of this set allows the authors to apply the Birkhoff Ergodic Theorem over $\Omega_\infty^+$.  We thus have the following proposition:
\begin{proposition}[AdaBoost is Ergodic over $\Omega_\infty^+$]
    The average over any AdaBoost sequence starting at $w_1 \in \Omega_\infty^+$ converges. More precisely,
    \begin{equation}
        \forall w_1\in \Omega_\infty^+, \forall f \in L^1(\mu), \frac{1}{T}\sum_{t=0}^{T-1} f(\A^t(w_1)) \underset{T\to\infty}{\to} L(f)\in \R
    \end{equation}
\end{proposition}
Again, the proof is technical, and does not present much interest for our purpose. Finally, the authors show a last big result which they prove under different hypothesis. 
Here as well, we won't go too much into the details here, but the theorem resembles to the following:
\begin{theorem}[AdaBoost is Ergodic and Converges to a Cycle]\label{thm:cycle}
    Let $\A^\tau$ a specific sequence of functions that converge towards the AdaBoost update $\A$ uniformly over $\Delta_m$. (In practice, 
    those functions are explicit in the paper, and the author show the uniform convergence). Then,
    \begin{itemize}
        \item $(w_t)$ converges in finite time to a cycle in $\Delta_m^+$ of period $p$.
        \item The AdaBoost system is ergodic.
        \item Let $T_0$ be the first time at which $w_t$ enters the cycle. Then, for any $f \in L^1(\mu)$, we have that:
            \begin{equation}
                \frac{1}{T}\sum_{t=0}^{T - 1} f(\A^t(w_1)) \underset{T\to\infty}{\to} \frac{1}{p} \sum_{t=0}^{p-1} f(\A^{T_0+t}(w_1))
            \end{equation}
    \end{itemize}
\end{theorem}

That shows that after a high number of iterations, AdaBoost becomes perfectly cyclic. That confirms the intuition we can
have Sec.~\ref{subsec:Diversity}. 

However, this theorem is not completely satisfying. Indeed, this demonstrates that AdaBoost cycles w.r.t.~the weights (thus w.r.t.~the estimators $h$), but
the cycle may be very long, and take a lot of iterations to be reached. In practice, we want to know if AdaBoost cycles w.r.t.~the estimators $h$. The problem is 
that a cycle w.r.t.~the estimators $h$ may not imply a cycle w.r.t.~the weights. Indeed, the weights are not unique for a given estimator $h$. This is what 
motivated a discussion that we present in a future paper we will publish in a short time.

\section{Conclusion \& Acknowledgements}

\subsection{Conclusion}

The primary objective of this paper has been to provide a comprehensive and expansive overview of AdaBoost, exploring the diverse interpretations and facets 
that extend beyond its initial introduction as a PAC learning algorithm. We have uncovered various ways to perceive and understand AdaBoost, highlighting the 
significance of comprehending each interpretation to attain a comprehensive understanding of the algorithm's inner workings. In particular, the ergodic dynamics 
of AdaBoost have emerged as a compelling avenue of research, offering a new perspective that can foster theoretical advancements and shed light on its long-term behavior.

Although AdaBoost's iterative dynamics are explicitly accessible at each step, comprehending its overall behavior on a global scale remains a challenging task. 
Consequently, studying the ergodic dynamics of AdaBoost has become an active and fruitful research domain for the past 30 years. By viewing AdaBoost as an ergodic dynamical system, 
novel theoretical frameworks and perspectives can be developed, enhancing our comprehension of its convergence properties, generalization bounds, and optimization landscape. 
This approach opens doors to explore the interplay between AdaBoost and related fields, such as deep learning, where phenomena like double descent have garnered significant attention.

Remarkably, the presence of a similar behavior to the double descent phenomenon was observed soon after AdaBoost's introduction, yet it has not received extensive study except for a few notable papers. 
Given the recent surge of interest in double descent within the realm of deep learning, establishing connections and parallels between the dynamics of AdaBoost iterations and deep 
learning can significantly advance our understanding of both domains. By bridging these two fields of research, we can gain valuable insights into the intricate dynamics governing 
AdaBoost, leading to a deeper appreciation of its predictive power and potential applications.

This paper aims to serve various purposes for its readers. For those unfamiliar with AdaBoost, it offers a clear and concise introduction to the algorithm, elucidating 
its multiple interpretations and shedding light on its fundamental principles. By presenting the algorithm's different perspectives in an accessible manner, readers can 
develop a solid foundation in AdaBoost and its significance within the broader machine learning landscape. Furthermore, experienced readers will find value in the paper's 
ability to establish connections and unify diverse views of AdaBoost, presenting a cohesive and comprehensive understanding of the algorithm. This synthesis of perspectives 
provides a launching pad for future research endeavors centered around AdaBoost's dynamics, enabling scholars to delve deeper into its behavior and explore novel research directions.

In conclusion, this paper has endeavored to provide a thorough exploration of AdaBoost, transcending its initial formulation as a PAC learning algorithm. By unifying 
the diverse interpretations and facets of AdaBoost, we have unveiled the intriguing concept of ergodic dynamics, which holds promise for advancing theoretical frameworks 
and deepening our comprehension of the algorithm's behavior. We have also highlighted the significance of investigating the double descent phenomenon and establishing 
connections between AdaBoost other fields. By presenting AdaBoost in a multi-faceted manner, this paper caters to both novice and experienced readers, fostering a 
comprehensive understanding of the algorithm and paving the way for future research endeavors in AdaBoost's dynamics and its broader implications in machine learning.

\subsection{Acknowledgements}
I would like to deeply thank Nicolas Vayatis and Argyris Kalogeratos for the time and effort they put in this project. I would also like to thank the 
IDAML Chair of Centre Borelli and its private partners for the funding of this project, without which this work would not have been possible.

\printbibliography

\begin{appendices}
    \section{How to choose the base space of estimators?}\label{sec:How to choose the base space of estimators?}
From original to most modern versions of AdaBoost, it is always mentioned that AdaBoost, as any other boosting method,
should be executed on a set of \textit{weak estimators}. The definition of a weak estimator is not precise. In the original
papers, boosting was considered as a PAC learning algorithm meaning that each estimator had at least a slightly better 
performance as a purely random estimator. Mathematically, this can write, for a classification problem with $K$ classes, as follows:\newline
If the observation $x$ has the (true) label $k$, then there exists $\delta > 0$ such that
\begin{equation}
    \P(y=k|h(x)=k) \geq \P_{prior}(y=k) + \delta
\end{equation}
where $\P_{prior}$ is a fixed prior that we chose according to our knowledge. For instance,
we could fix $\P_{prior}(y=k) = \frac{\sum_{i=1}^m \mathbbm{1}_{y_i=k}}{m}$.

\subsection{Too weak or too strong estimators}

However, would boosting be relevant if our estimators were already strong in that sense? To properly see this, let's take
$\H_d(\X)$ the set of decision trees of depth $d$ over the set $\X$. Suppose we have $m$ observations in $\X$. It is
of course possible to build a tree $h\in \H_m(\X)$ which achieves 100\% accuracy on the set. Nonetheless, AdaBoost is not
even designed to generate the next iteration of such classifier, as it is suppose to have an error $\epsilon > 0$. But even considering
decision trees of depth sufficiently small to ensure that none of them can classify the whole set, boosting can lose sense. Indeed, 
we can have two opposite phenomenas:
\begin{itemize}
    \item \textit{Too weak} estimators
    \item \textit{Too strong} estimators
\end{itemize}
How can we have \textit{too weak} estimators? Considering a classification task with $K>>2$ classes, we need to have complex enough
weak estimators for boosting to properly work. If $\H_1(\X)$ is our set of weak estimators, each of the tree of this set will
achieve a very low accuracy, and probably way less than any random guesser if the data is complex enough to be non-linearly seperable
for instance. 

Now, we can have estimators that are \textit{too strong} as well. Indeed, if each of the tree in the boosting sequence
achieves an accuracy which is close to perfect accuracy, the benefits from boosting methods vanish automatically. Of course,
it still depends on what you aim for in your problem, but using a boosting method to fit $100$ trees of depth $10$ on a complex
task can require much more time than training a single strong classifier. Plus, this will not likely prevent overfitting as
the sequence of decision trees will not vary a lot (because each tree achieves very satisfying accuracy on the train set), meaning
that the boosting classifier will truly use only a few different decision trees when you expect it to use tens or hundreds different
ones to aggregate the result.

\subsection{How to detect too good or too weak estimators?}
The easier to detect is when your estimator are too strong. It is also probably the one that is the most likely to
happen. Indeed, in that case, you can observe several things running boosting:
\begin{itemize}
    \item Each one of the estimator in the sequence achieves high accuracy on the train set.
    \item The \textit{similarity} between the estimators is too high.
    \item The difference between the precision of each estimator and the precision of the aggregated estimator is small.
\end{itemize}
Let's take an example here. Consider the \href{https://www.kaggle.com/datasets/zalando-research/fashionmnist}{Fashion MNIST dataset}. We consider several sets of estimators:
$\H_1(\X),\dots,\H_{15}(\X)$. For each depth, we run AdaBoost for $T=\{10,20,\dots,100\}$ iterations. 
We then compare the difference between the mean accuracy of each decision tree and the ensemble model. Also, we compute the
similarity between each estimator of the sequence.

\begin{figure}
    \centering
    \includegraphics[width=0.8\textwidth]{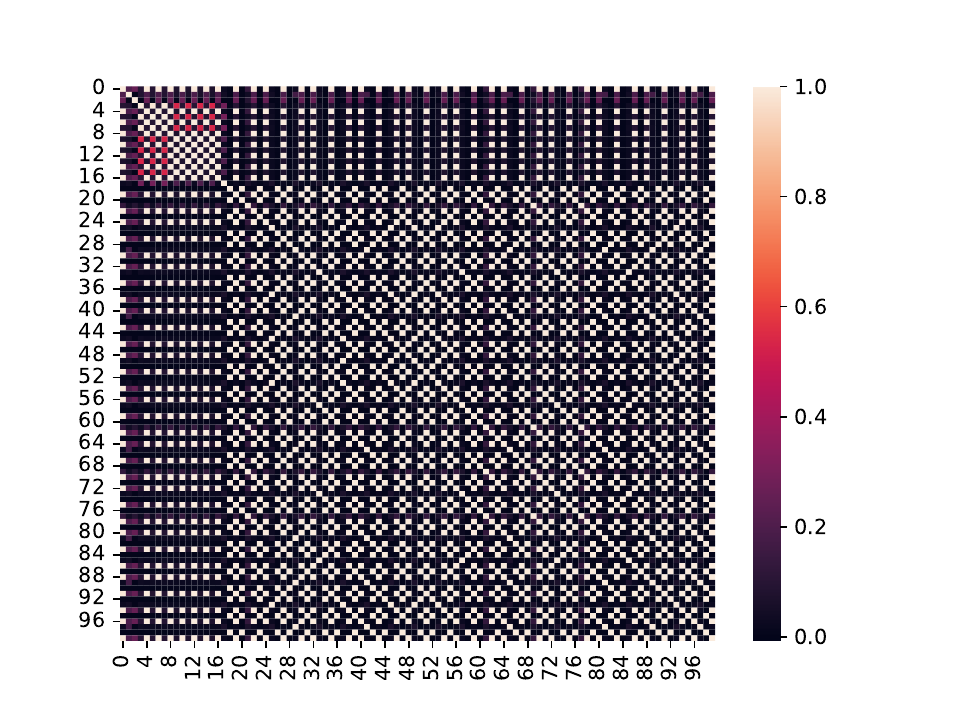}
    \caption{Similarity matrix for depth $1$. In $(i,j)$, the similarity measure between classifier
    $i$ and $j$ measured by kappa statistic. We recognize here the cyclic patterns due to the lack of
    complexity of our set of estimators.}
  \label{fig:Sim_matrix_depth1}
\end{figure}

\begin{figure}
    \centering
    \includegraphics[width=0.8\textwidth]{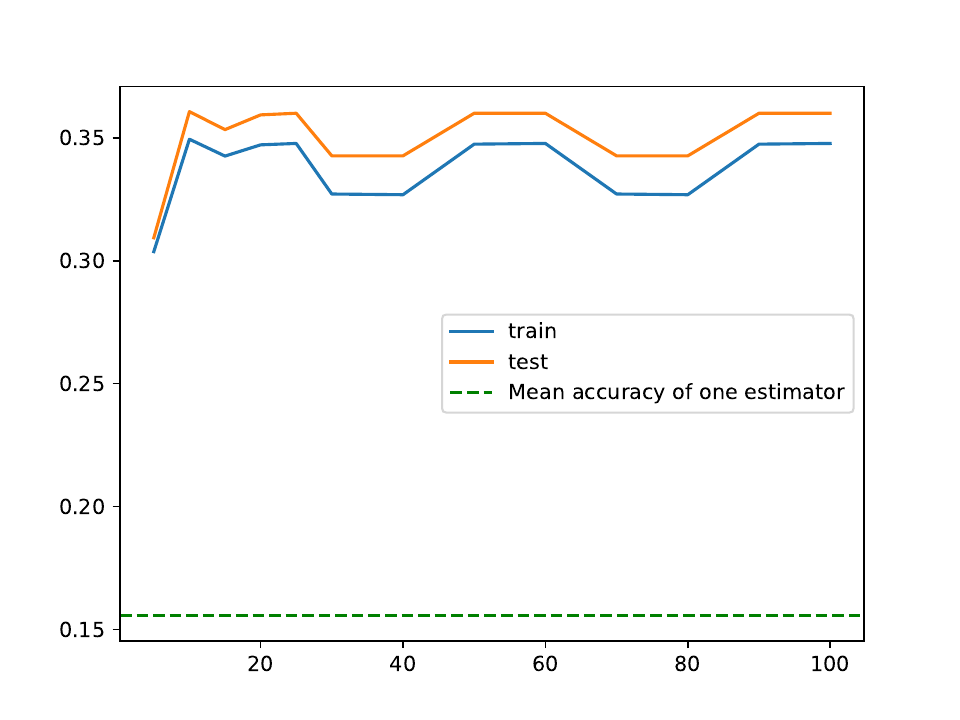}
    \caption{Accuracy for depth $1$. On the $x$-axis, the number of estimators kept to build the ensemble, 
    and on the $y$-axis, the accuracy of the considered estimator.}
  \label{fig:acc_depth1}
\end{figure}

\begin{figure}
    \centering
    \includegraphics[width=0.8\textwidth]{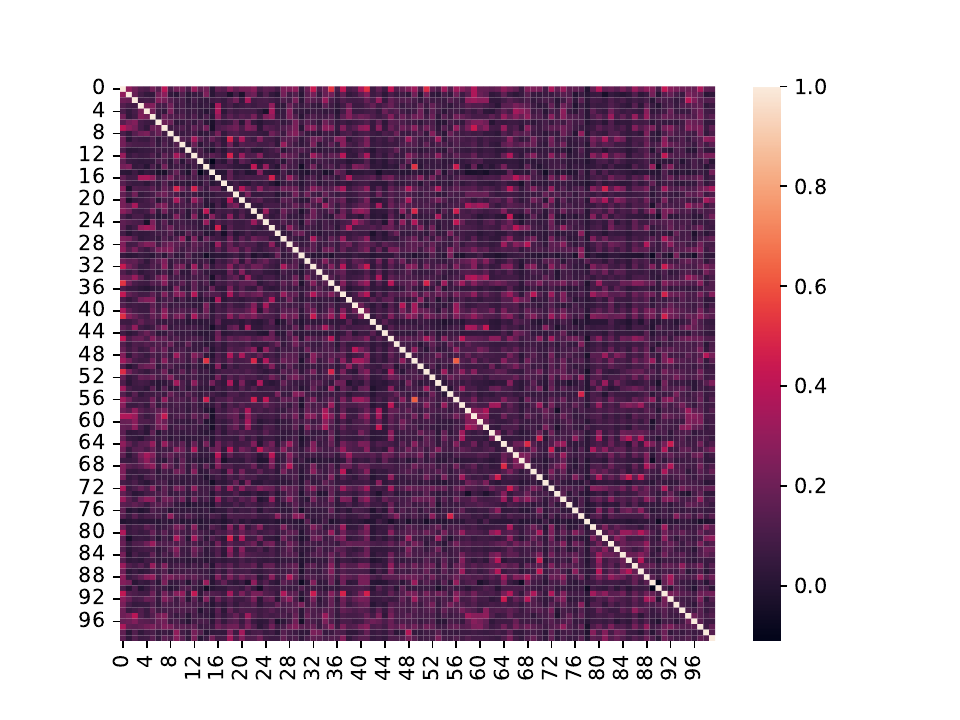}
    \caption{Similarity matrix for depth $4$. In $(i,j)$, the similarity measure between classifier
    $i$ and $j$ measured by kappa statistic. Here, the produced sequence is very diverse and all estimators
    are sufficiently different from each other to have a true benefit from boosting.}
  \label{fig:Sim_matrix_depth4}
\end{figure}

\begin{figure}
    \centering
    \includegraphics[width=0.8\textwidth]{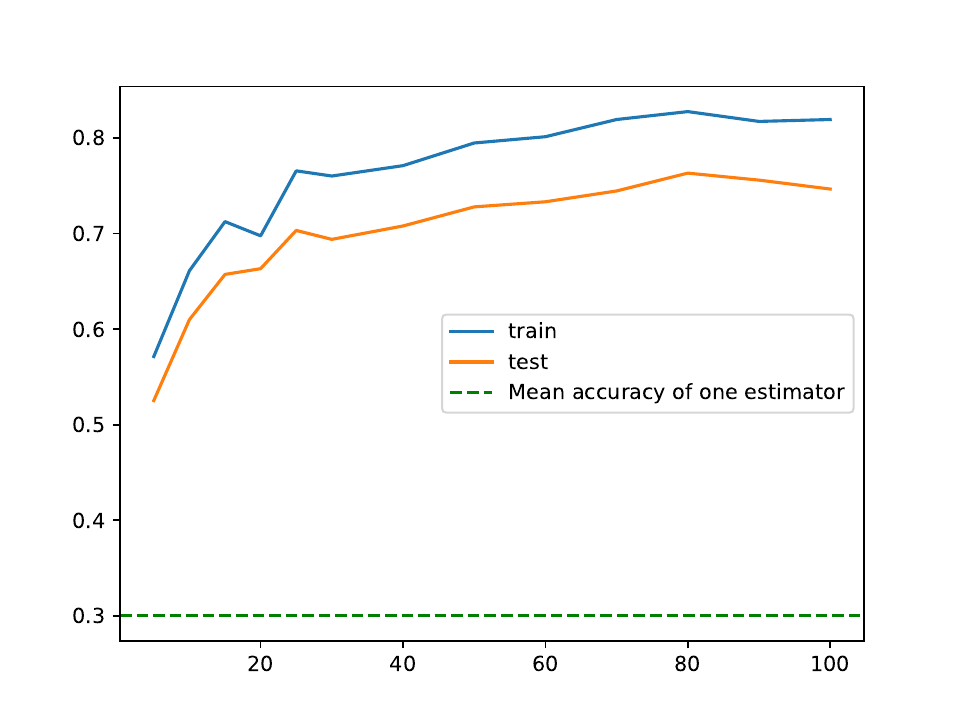}
    \caption{Accuracy for depth $4$. On the $x$-axis, the number of estimators kept to build the ensemble, 
    and on the $y$-axis, the accuracy of the considered estimator.}
  \label{fig:acc_depth4}
\end{figure}

\begin{figure}
    \centering
    \includegraphics[width=0.8\textwidth]{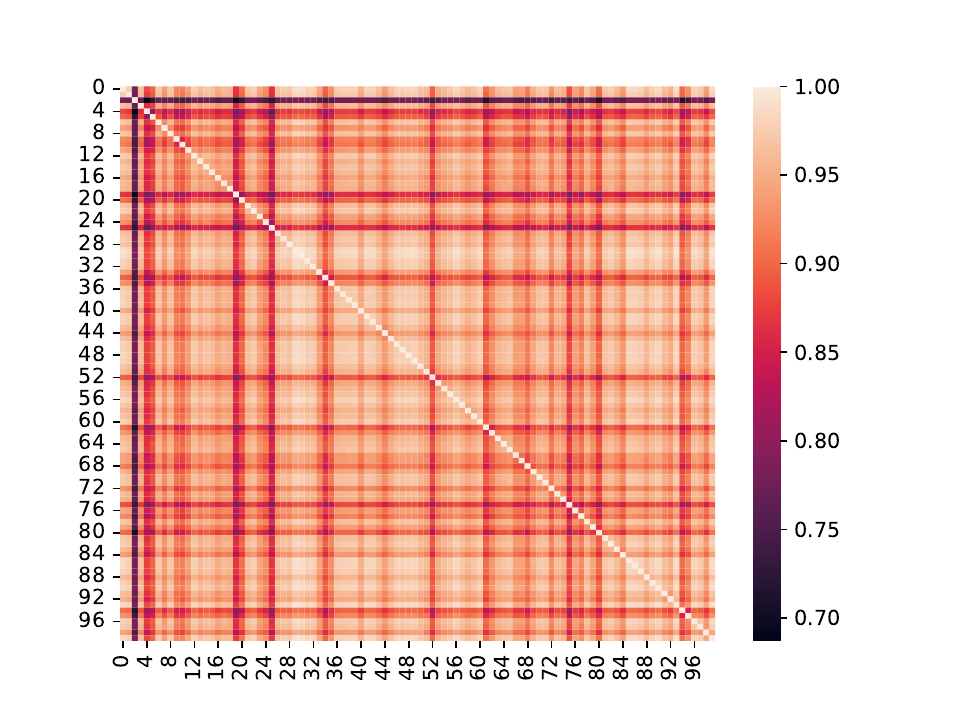}
    \caption{Similarity matrix for depth $15$. In $(i,j)$, the similarity measure between classifier
    $i$ and $j$ measured by kappa statistic. As we can see, the similarity is very high between each
    estimator, meaning that the boosting sequence is not diverse enough because the estimators are too strong.}
  \label{fig:Sim_matrix_depth15}
\end{figure}

\begin{figure}
    \centering
    \includegraphics[width=0.8\textwidth]{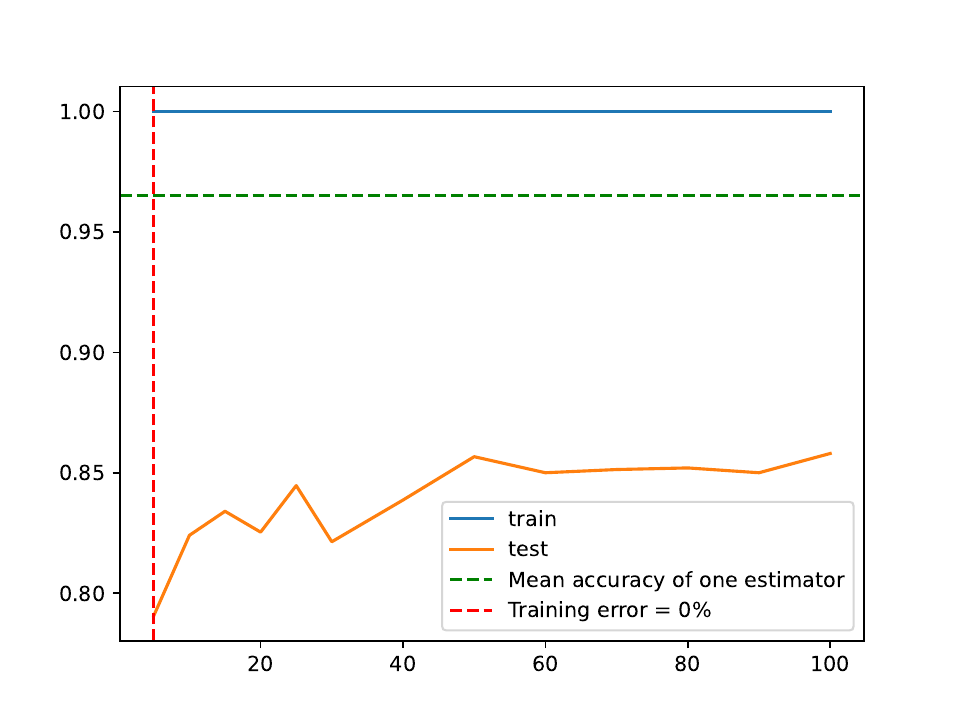}
    \caption{Accuracy for depth $15$. On the $x$-axis, the number of estimators kept to build the ensemble, 
    and on the $y$-axis, the accuracy of the considered estimator.}
  \label{fig:acc_depth15}
\end{figure}

On the Fig.~\ref{fig:Sim_matrix_depth1}, \ref{fig:acc_depth1}, \ref{fig:Sim_matrix_depth4}, \ref{fig:acc_depth4}, \ref{fig:Sim_matrix_depth15},
\ref{fig:acc_depth15}, \ref{fig:mean_similarity}, we see that if we take a depth of $1$, the estimators are too weak and we cannot learn properly as the algorithm
stucks itself in a cycle to quickly. However, we have the opposite with depth $15$ for which boosting seems quite useless, and all estimators
seem to be the same. A good set of classifiers here is thus trees of depth $4$ or $5$, for example.
We can also observe that the more more complex the set of estimator is, the more similar the sequence of estimator is. 

\begin{figure}
    \centering
    \includegraphics[width=0.8\textwidth]{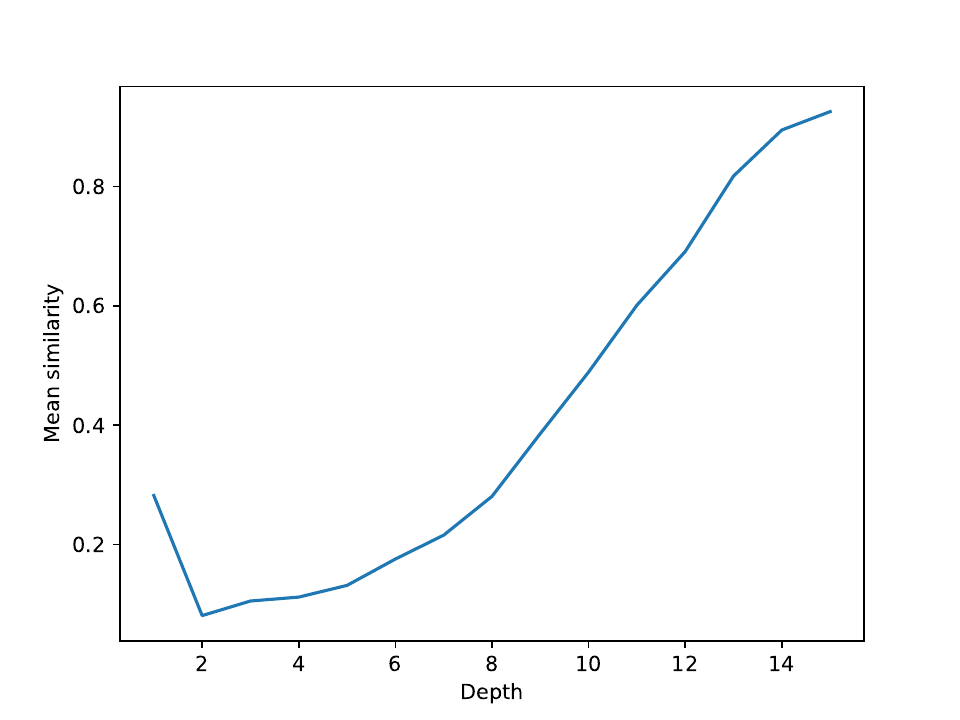}
    \caption{Mean similarity inside the sequence of estimators. We notice that the more complex the set of estimator is, the more
    similar the sequence of estimators become.}
  \label{fig:mean_similarity}
\end{figure}

\end{appendices}

\end{document}